%% file: official_preprint.tex
\documentclass[10pt,twocolumn,letterpaper]{article}

\usepackage{iccv}
\usepackage{times}
\usepackage{epsfig}
\usepackage{graphicx}
\usepackage{amsmath}
\usepackage{amssymb}
\usepackage{algorithm}
\usepackage{algpseudocode}
\usepackage{url}
\usepackage{caption} 
\usepackage[english]{babel}
\usepackage{blindtext}
\usepackage{amsthm}
\usepackage{graphicx}
\usepackage{wrapfig}
\usepackage{array}
\usepackage{caption}
\usepackage{subcaption}
\usepackage{bm}
\usepackage{xcolor}
\usepackage{enumitem}
\usepackage{tablefootnote}
\usepackage{amsmath}
\usepackage{mathtools}
\usepackage{xcolor}
\usepackage{float}
\usepackage[symbol]{footmisc}
\usepackage{tabularray}
\usepackage{amsthm}
\usepackage{thmtools, thm-restate}

\input{math_commands.tex}

\newcommand{\modelname}{FDRL}

\algrenewcommand\algorithmicindent{0.5em}

\usepackage[pagebackref=true,breaklinks=true,letterpaper=true,colorlinks,bookmarks=false]{hyperref}

\iccvfinalcopy %

\def\iccvPaperID{9943} %

\ificcvfinal\fi

\begin{document}

\title{Generative Modeling with Flow-Guided Density Ratio Learning}

\author{\large Alvin Heng\textsuperscript{1,*}, \large Abdul Fatir Ansari\textsuperscript{1,2,$\dagger$}, \large Harold Soh\textsuperscript{1,3}\\
\normalsize \textsuperscript{1}Dept. of Computer Science, National University of Singapore\\
\normalsize \textsuperscript{2}AWS AI Labs\\
\normalsize \textsuperscript{3}Smart Systems Institute, National University of Singapore
}

\maketitle
\ificcvfinal\thispagestyle{empty}\fi

\begin{abstract}
We present Flow-Guided Density Ratio Learning (\modelname{}), a simple and scalable approach to generative modeling which builds on the stale (time-independent) approximation of the gradient flow of entropy-regularized $f$-divergences introduced in recent work. Specifically, the intractable time-dependent density ratio is approximated by a stale estimator given by a GAN discriminator. This is sufficient in the case of sample refinement, where the source and target distributions of the flow are close to each other. However, this assumption is invalid for generation and a naive application of the stale estimator fails due to the large chasm between the two distributions. \modelname{} proposes to train a density ratio estimator such that it learns from progressively improving samples during the training process. We show that this simple method alleviates the density chasm problem, allowing \modelname{} to generate images of dimensions as high as $128\times128$, as well as outperform existing gradient flow baselines on quantitative benchmarks. We also show the flexibility of \modelname{} with two use cases. First, unconditional \modelname{} can be easily composed with external classifiers to perform class-conditional generation. Second, \modelname{} can be directly applied to unpaired image-to-image translation with no modifications needed to the framework. Our code is publicly available at \href{https://github.com/clear-nus/fdrl}{\texttt{https://github.com/clear-nus/fdrl}}.

\end{abstract}

\footnotetext[1]{
Correspondence at: \texttt{\{alvinh,harold\}@comp.nus.edu.sg}}
\footnotetext[2]{
Work done while at the National University of Singapore, prior to
joining Amazon.}

\section{Introduction}

\begin{figure*}
    \centering
    \includegraphics[width=\textwidth]{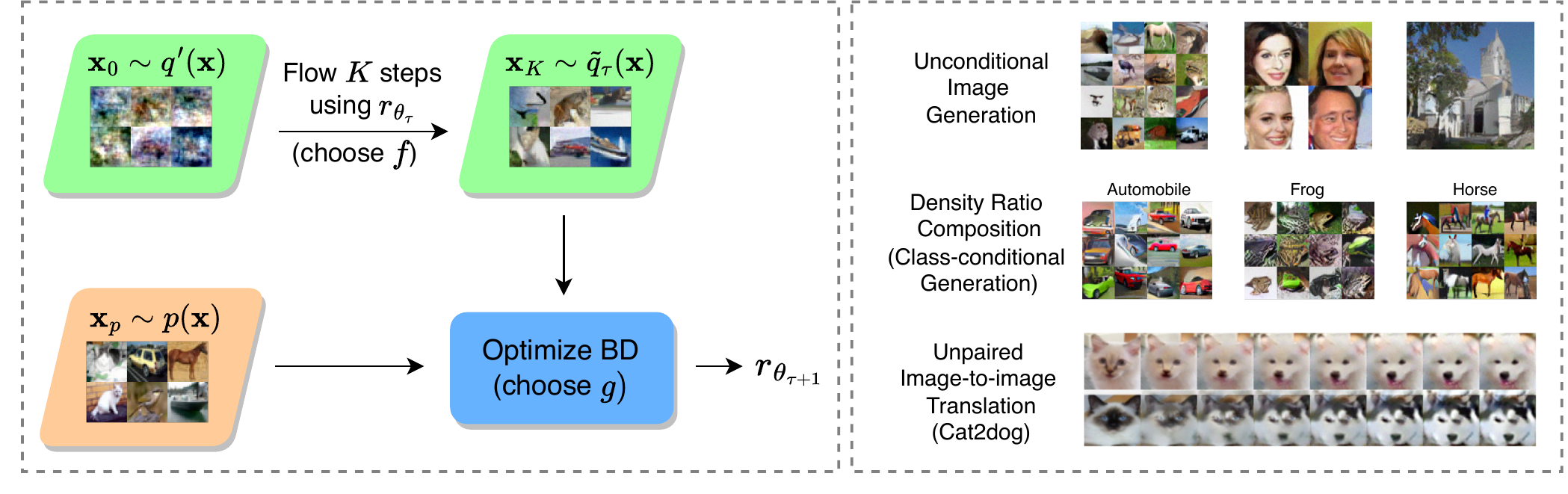}
    \caption{Left: Illustration of \modelname{}'s training setup for the $\tau$ training iteration. For clarity, we emphasize the choices of the $f$-divergence and $g$, the Bregman divergence function, as part of training. Right: the various applications of \modelname{}, ranging from unconditional image generation to class-conditional generation by composition with external classifiers and unpaired image-to-image translation. }
    \label{fig:main}
\end{figure*}

Deep generative modeling has made significant strides in recent years. Advances in techniques such as generative adversarial networks~\cite{goodfellow2014generative} and diffusion models~\cite{ho2020denoising, song2020score} have enabled the generation of high-quality, photorealistic images and videos. Recently, gradient flows have emerged as an interesting alternative approach to generative modeling. Unlike the aforementioned techniques, gradient flows have a unique interpretation where the objective is to find the path of steepest descent between two distributions. 
Further, where diffusion models typically require a tractable terminal distribution such as Gaussian, the source and target distributions in gradient flows can be varied for different tasks in generative modeling. For instance, the gradient flow performs data synthesis when the source distribution is a simple prior and the target is the data distribution. When the two distributions are from different domains of the same modality, the gradient flow performs domain translation. 

Among gradient flow methods, Wasserstein gradient flows (WGF) have recently become a popular specialization for a variety of applications \cite{mokrov2021large, fan2021variational, ansari2020refining, gao2022deep}. WGFs model the gradient dynamics on the space of probability measures with respect to the Wasserstein metric; these methods aim to construct the optimal path between two probability measures --- a source distribution $q(\rvx)$ and a target distribution $p(\rvx)$, where the notion of optimality refers to the path of steepest descent in Wasserstein space. 

Despite the immense potential of gradient flow methods in generative modeling~\cite{ansari2020refining, gao2022deep, fan2021variational}, applications to high-dimensional image synthesis remain limited, and its applications to other domains, such as class-conditional generation and image-to-image translation have been underexplored. A key difficulty is that existing methods to solve for the gradient flow often involve complex approximation schemes. For instance, the JKO scheme~\cite{jordan1998variational} necessitates estimation of the 2-Wasserstein distance and divergence functionals, while
particle simulation methods, which involve simulating an equivalent stochastic differential equation (SDE), requires solving for the time-dependent marginal distribution over the flow.

A recent work on adopting gradient flows for sample refinement, DG$f$low \cite{ansari2020refining}, proposes a stale (time-independent) estimate given by the discriminator of a pretrained GAN, as the density ratio estimate necessary in the corresponding SDE. While this dramatically simplifies the particle simulation approach, in practice this only works when $q(\rvx)$ and $p(\rvx)$ are close together, such as between the GAN generated images and the true data distribution. When the two distributions are far apart, such as the case of image synthesis, where $q(\rvx)$ is a simple prior distribution, estimating the density ratio $q(\rvx)/p(\rvx)$ becomes a trivial problem due to the large density chasm \cite{rhodes2020telescoping}, and use of this naive stale estimator in simulating the SDE fails to generate realistic images.

We therefore pose the question: \emph{is it possible to adopt a simple gradient flow scheme for image generation driven by density ratio learning?} We answer this in the affirmative by proposing Flow-Guided Density Ratio Learning (FDRL), a new training approach to address the problem of estimating the density ratio between the prior and the data distribution, by progressively training a density ratio estimator with evolving samples. FDRL operates exclusively in the data space, and does not require training additional generators, unlike related particle methods \cite{gao2022deep}. Compared to gradient flow baselines, we achieve the best quantitative scores in image synthesis and are the first to scale to image dimensions as high as $128 \times 128$. In addition, we show that our framework can be seamlessly applied to two other critical tasks in generative modeling: class-conditional generation and unpaired image-to-image translation, which have been under-explored in the context of gradient flow methods. 

\section{Background}
We provide a brief overview of gradient flows and density ratio estimation. For a more comprehensive introduction to gradient flows, please refer to~\cite{santambrogio2017euclidean}. A thorough overview of density ratio estimation can be found in~\cite{sugiyama2012density}.

\paragraph{Wasserstein Gradient Flows.}
Consider Euclidean space equipped with the familiar $L_2$ distance metric $(\mathcal{X}, \|\cdot\|_2)$. Given a function $ F: \mathcal{X} \rightarrow \mathbb{R}$, the curve $\{\rvx (t)\}_{t\in\mathbb{R}^+}$ that follows the direction of steepest descent is called the \textit{gradient flow} of $F$:
\begin{equation}\label{eq:gf_euclidean}
    \rvx'(t) = -\nabla F(\rvx(t)).
\end{equation}
In generative modeling, we are interested in sampling from the underlying \textit{probability distribution} of a given dataset. Hence, instead of Euclidean space, we consider the space of \textit{probability measures} with finite second moments equipped with the 2-Wasserstein metric $(\mathcal{P}_2(\Omega), \mathcal{W}_2)$. Given a functional $\mathcal{F} : \mathcal{P}_2(\Omega) \rightarrow \mathbb{R}$ in the 2-Wasserstein space, the gradient flow of $\mathcal{F}$ is the steepest descent curve of $\mathcal{F}$. Such curves are called Wasserstein gradient flows.

\paragraph{Density Ratio Estimation via Bregman Divergence.}
Let $q(\rvx)$ and $p(\rvx)$ be two distributions over $\mathcal{X} \in \mathbb{R}^d$ where we have access to i.i.d samples $\rvx_q \sim q(\rvx)$ and $\rvx_p \sim p(\rvx)$. The goal of density ratio estimation (DRE) is to estimate the true density ratio $r^*(\rvx) = \frac{q(\rvx)}{p(\rvx)}$ based on samples $\rvx_q$ and $\rvx_p$.

We will focus on density ratio fitting under the Bregman divergence (BD), which is a framework that unifies many existing DRE techniques~\cite{sugiyama2012density, sugiyama2012densityratio}. Let $g : \mathbb{R}^+ \rightarrow \mathbb{R}$ be a twice continuously differentiable convex function with a bounded derivative. The BD seeks to quantify the discrepancy between the estimated density ratio $r_\theta$ and the true density ratio $r^*$:
\begin{multline}\label{eq:bd_true}
    BD_g(r^*||r_\theta) = \mathbb{E}_{p}[g(r^*(\rvx)) - g(r_\theta(\rvx)) + \partial g(r_\theta(\rvx))r_\theta(\rvx)] \\ - \mathbb{E}_{q} [ \partial g(r_\theta(\rvx))].
\end{multline}
In practice, we estimate the expectations in Eq. \ref{eq:bd_true} using Monte Carlo samples. We can also drop the term $\mathbb{E}_p[g(r^*(\rvx))]$ during optimization as it does not depend on the model $r_\theta$.
The minimizer of Eq. \ref{eq:bd_true}, which we denote $\theta^*$, satifies $r_{\theta^*}(x)=r^*(\rvx)=q(\rvx)/p(\rvx)$. Moving forward, we will use the hatted symbol $\mathbb{\widehat{E}}_p$ to refer to Monte Carlo estimates for ease of notation. We consider two common instances of $g$ in this paper, namely $g(y)=\frac{1}{2}(y-1)^2$ and $g(y)=y\log y - (1+y)\log(1+y)$, which correspond to the Least-Squares Importance Fitting (LSIF) and Logistic Regression (LR) objectives. We provide the full form of the BD for these choices of $g$ in Appendix \ref{sec:app_pairing}. 

\section{Discriminator Gradient Flow}
Ansari et al. propose DG$f$low~\cite{ansari2020refining} for sample refinement by simulating the SDE:
\begin{equation}\label{eq:sde}
    d\rvx_t = -\nabla_\rvx f'(q_t(\rvx_t)/p(\rvx_t))dt + \sqrt{2\gamma}d\rvw_t,
\end{equation}
where $d\rvw_t$ denotes the standard Wiener process. $p(\rvx)$ represents the target distribution of refinement, while $q_t(\rvx)$ represents the time-evolved source distribution $q_0(\rvx)$. In theory, as $t \rightarrow \infty$, the stationary distribution of the particles $\rvx_t$ from Eq.~\ref{eq:sde} converges to the target distribution $p(\rvx)$. The SDE is the equivalent particle system of the corresponding Fokker-Planck equation, which is the gradient flow of the functional 
\begin{equation}\label{eq:erfd}
    \mathcal{F}^f_p(q) = \int p(\rvx) f(q(\rvx)/p(\rvx)) d\rvx + \gamma \int q(\rvx) \log q(\rvx)d\rvx,
\end{equation}
where $f : \mathbb{R}^+ \rightarrow \mathbb{R}$ is a twice-differentiable convex function with $f(1)=0$. Eq.~\ref{eq:sde} can therefore be understood as the flow which minimizes Eq.~\ref{eq:erfd}. The first term in Eq.~\ref{eq:erfd} is the $f$-divergence, which measures the discrepancy between $q(\rvx)$ and $p(\rvx)$. Popular $f$-divergences include the Kullback-Leibler (KL), Pearson-$\chi^2$ divergence and Jensen-Shannon (JS) divergence. Meanwhile, the second term is a differential entropy term that improves expressiveness of the flow. We list the explicit forms of $f$ that will be considered in this work in Table \ref{tab:f-div} in the Appendix.

In DG$f$low, the source distribution is the distribution of samples from the GAN generator, while the target distribution is the true data distribution. The time-varying density ratio $q_t(\rvx)/p(\rvx)$ is approximated by a stale estimate $r_\theta(\rvx)$
\begin{equation}\label{eq:sde_stale}
    d\rvx_t = -\nabla_\rvx f'(r_\theta(\rvx_t))dt + \sqrt{2\gamma}d\rvw_t,
\end{equation}
where $r_\theta$ is represented by the GAN discriminator. The assumption underlying this approach is that the distribution of generated samples from the GAN is sufficiently close to the data distribution, so that the density ratio $q_t(\rvx)/p(\rvx)$ is approximately constant throughout the flow.

As we are unable to compute the marginal distributions $q_t(\rvx)$ or estimate the time-dependent density ratio $q_t(\rvx)/p(\rvx)$ in Eq.~\ref{eq:sde} a priori, we are motivated to extend the stale formulation of DG$f$low from refinement to the more difficult problem of synthesis, i.e., generating samples from noise. 
We show that in such a scenario, the stationary distribution of Eq.~\ref{eq:sde_stale} shares the same MLE as the data distribution. 
\begin{restatable}{lemma}{mainlemma}\label{thm:main}
Let $\gamma=1$ and assume that $q(\rvx)\sim U[a,b]$, where $a,b$ are chosen appropriately (e.g., [-1,1] for image pixels). Then the stationary distribution of Eq.~\ref{eq:sde_stale}, $\rho_\infty(\rvx)$, has the same maximum likelihood estimate as $p(\rvx)$, $\argmax_\rvx \rho_\infty(\rvx) = \argmax_\rvx p(\rvx)$.
\end{restatable}
Based on Lemma~\ref{thm:main}, we hypothesize that samples drawn from Eq.~\ref{eq:sde_stale} have high likelihoods under the data distribution. This motivates us to train a stale estimator $r_\theta(\rvx)$ where $q(\rvx)$ is a simple prior and $p(\rvx)$ is the true data distribution.

\subsection{Where The Stale Estimator Breaks Down}\label{sec:stale_breaks}

We thus wish to modify the source distribution such that $q(\rvx)$ is a simple prior, such as a Gaussian, while keeping $p(\rvx)$ to be the true data distribution. Given samples from $q(\rvx)$ and $p(\rvx)$, we can leverage the Bregman divergence to train a stale estimator and simulate Eq.~\ref{eq:sde_stale}. However, we observe empirically that this fails catastrophically as the density ratio estimation problem becomes trivial and the Bregman loss diverges. We convey the key intuitions with a toy example of simple 2D Gaussians.

Consider an SDE of the form 
\begin{equation}
    d\rvx_t = -\nabla_\rvx f'(q(\rvx_t)/p(\rvx_t))dt + \sqrt{2\gamma}d\rvw_t,
\end{equation}
where $q(\rvx)$ and $p(\rvx)$ are both simple 2D Gaussian densities. This differs from Eq.~\ref{eq:sde} as the distribution $q(\rvx_t)$ is time-independent, and fixed as the source Gaussian distribution. We would like to investigate how such an SDE, with $q/p$ represented by a stale estimator $r_\theta$, behaves in transporting particles from $q(\rvx)$ to $p(\rvx)$.

Although we have access to the analytical form of the density ratio $q(\rvx)/p(\rvx)$ in this example, we choose to train a multilayer perceptron (MLP) to estimate the density ratio using the Bregman divergence Eq. \ref{eq:bd_true}. This is to evaluate the general case where we may only have samples and not the analytic densities. We investigate two scenarios by varying the distance between $q$ and $p$, as measured by their KL divergence (or more generally, the $f$-divergences). $q(\rvx)$ is fixed at  $\mathcal{N}(\mathbf{0}, 0.1\mathbb{I})$, where bold numbers represent 2D vectors. We set $p(\rvx) \sim \mathcal{N}(\mathbf{1}, 0.1\mathbb{I})$ in the first case and $p(\rvx) \sim \mathcal{N}(\mathbf{6}, 0.1\mathbb{I})$ in the second. 

We first train separate MLPs to estimate the density ratio, using $\rvx_q$ and $\rvx_p$ drawn from the respective distributions. After convergence, we simulate a finite number of steps of Eq.~\ref{eq:sde_stale} on samples drawn from $q(\rvx)$, with the density ratio given by the MLP. We characterize the transition of particles through Eq.~\ref{eq:sde_stale} as parameterized by a generalized kernel
\begin{equation}\label{eq:kernel}
    \Tilde{q}(\rvx) = \int q(\rvx') M_\theta(\rvx|\rvx') d\rvx',
\end{equation}
where $M_\theta$ is the transition kernel parameterized by the MLP.
The results are shown in Fig.~\ref{fig:stale_exps}. In (a), where the two distributions are close together, the particles from $q(\rvx)$ are successfully transported to $p(\rvx)$. Meanwhile in (b), even after flowing for many more steps, the particles have not converged to the target distribution. Flowing for more steps in (b) results in the particles drifting even further from the target (and out of the image frame). This suggests that the stale estimator is only appropriate in the case where the two distributions are not too far apart, which indeed corresponds to the assumption of DG$f$low that the GAN generated images are already close to the data distribution.

This observation can be attributed to two complementary explanations. Firstly, the Bregman divergence is optimized using samples in a Monte Carlo fashion; when the two distributions are separated by large regions of vanishing likelihood, there are insufficient samples in these regions for the model to learn an accurate estimate of the density ratio. Consequently, the inaccurate estimate causes the particles to drift away from the correct direction of the target as it crosses these regions, as seen in Fig.~\ref{fig:stale_no_works}.

The second explanation may be attributed to the density chasm problem~\cite{rhodes2020telescoping}. When two distributions are far apart, a binary classifier (which is equivalent to a density ratio estimator, see appendix Sec.~\ref{app:classifier}) can obtain near perfect accuracy while learning a relatively poor estimate of the density ratio. For example, in the second case above, the neural network can trivially separate the two modes with a multitude of straight lines, thus failing to estimate the density ratio accurately. 

\begin{figure}
     \centering
     \begin{subfigure}[b]{0.47\columnwidth}
         \includegraphics[width=\linewidth,trim={0cm 0cm 1cm 1cm},clip]{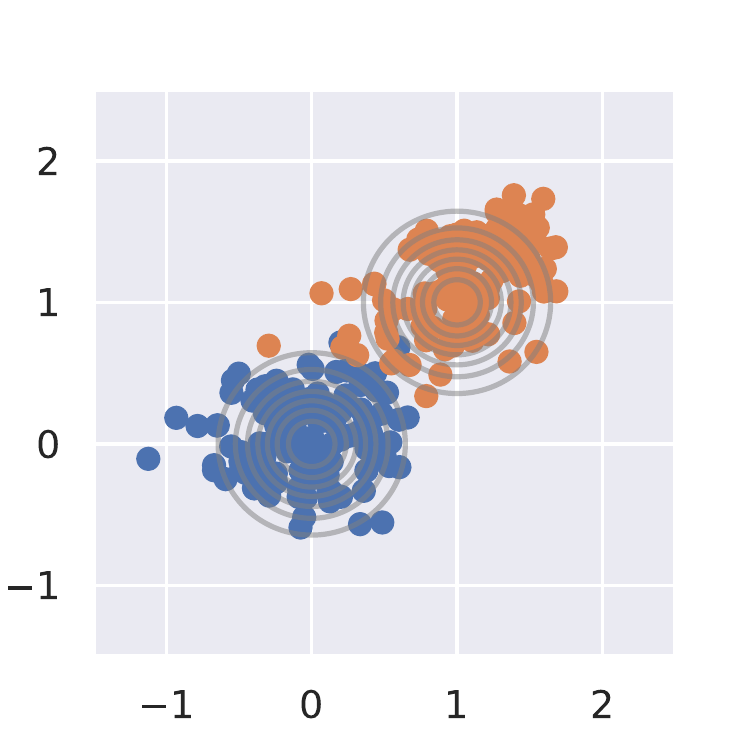}
         \caption{}
         \label{fig:stale_works}
     \end{subfigure}
     \hfill
     \begin{subfigure}[b]{0.47\columnwidth}
        \includegraphics[width=\linewidth,trim={0cm 0cm 1cm 1cm},clip]{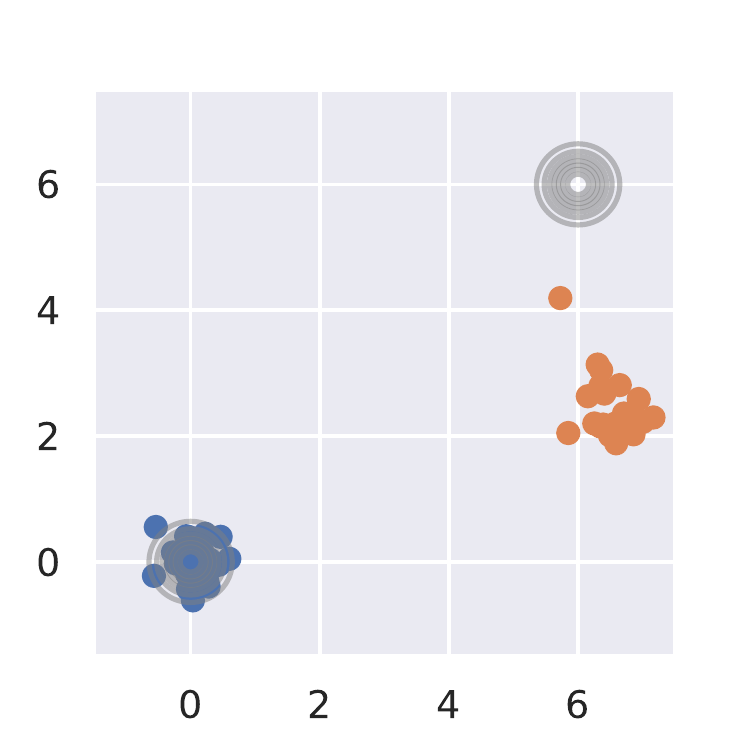}
         \caption{}
         \label{fig:stale_no_works}
     \end{subfigure}
        \caption{Toy experiments by simulating the flow Eq.~\ref{eq:sde_stale} with an MLP density ratio estimator. The source distribution $q(\rvx)$ is set as $\mathcal{N}(\mathbf{0}, 0.1\mathbb{I})$. Blue particles represent the source particles, and orange particles represent the same particles after flowing for $K$ steps. (a) We set $p(\rvx) \sim \mathcal{N}(\mathbf{1}, 0.1\mathbb{I})$ and flow for $K=15$ steps. (b) We set $p(\rvx) \sim \mathcal{N}(\mathbf{6}, 0.1\mathbb{I})$ and flow for $K=400$ steps. We can clearly see that particles in (a) have converged to the target distribution, while particles in (b) have not, demonstrating the density chasm problem.}
        \label{fig:stale_exps}
\end{figure}

The eventual goal of this work is the generative modeling of images, where we set our source distribution to be a simple prior distribution, and the target distribution to be the distribution of natural images. In such cases involving much higher-dimensional distributions, the density chasm problem will be dramatically more pronounced, causing the stale estimator to fail to generate realistic images. As mentioned previously, in early experiments where we attempt to learn a density ratio estimator with $q(\rvx)$ being a diagonal Gaussian or uniform prior and $p(\rvx)$ represented by samples of natural images (e.g. CIFAR10), the Bregman loss diverges early during training, leading to failed image samples.

\subsection{Data-Dependent Priors}\label{sec:ddp}
One approach to address the density chasm problem for images is to devise a prior $q(\rvx)$ that is as close to the target distribution as possible, while remaining easy to sample from. Inspired by generation from seed distributions with robust classifiers~\cite{santurkar2019image}, we first adopted a data-dependent prior (DDP) by fitting a multivariate Gaussian to the training dataset:
\begin{multline}
q(\rvx) = \mathcal{N}(\bm{\mu}_D, \bm{\Sigma}_D), \textrm{where} \\ \bm{\mu}_D = \mathbb{E}_D[\rvx], \ \ \ \bm{\Sigma}_D = \mathbb{E}_D[(\rvx-\bm{\mu}_D]^T (\rvx-\bm{\mu}_D)]
\end{multline}
where subscript $D$ represents the dataset. Samples from this distribution contain low frequency features of the dataset, such as common colors and outlines, which we visualize in Fig. \ref{fig:learned_prior} of the appendix. We proceeded to train a CNN with the Bregman divergence in the same manner as the toy experiments, with $\rvx_q$ drawn from the DDP and $\rvx_p$ from the dataset. However, we found that the Bregman loss still diverged early during training --- the DDP alone is insufficient to bridge the density chasm. In the next section, we describe a novel approach to cross the chasm using \emph{flow-guided} training, which when used with DDP achieves significantly better generative performance.

\section{Flow-Guided Density Ratio Learning}
 We propose FDRL, a training scheme that is able to scale to the domain of high-dimensional images, while retaining the simplicity of the stale formulation of DG$f$low. Instead of merely training $r_\theta$ to estimate $q(\rvx)/p(\rvx)$, our method proposes to estimate $\tilde{q}(\rvx)/p(\rvx)$ by flowing samples from $q(\rvx)$ at each training step with Eq.~\ref{eq:sde_stale}. This approch is based on observations that $\tilde{q}(\rvx)$ progressively approaches the target distribution as $r_\theta$ is being trained. In this formulation, $\tilde{q}(\rvx)$ is no longer static but changes with $r_\theta$ as $\theta$ is updated. Therefore, unlike the toy examples of Sec. \ref{sec:stale_breaks}, the density ratio estimator is being trained on samples that improve with training.

We formulate our training setup as follows. Consider a given training step $\tau$, where our model has parameters $\theta_{\tau}$ from the gradient update of the previous iteration. We switch notations slightly and denote $q'(\rvx)$ as our source distribution, which is a simple prior (such as the DDP). We propose to first draw samples $\rvx_0 \sim q'(\rvx)$, and simulate Eq. \ref{eq:sde_stale} for $K$ steps with the stale estimator $r_{\theta_{\tau}}$ using the Euler-Maruyama discretization:
\begin{equation}\label{eq:sde_euler_maruyama}
    \rvx_{k+1} = -\eta \nabla_\rvx f'(r_{\theta_{\tau}})(\rvx_k) + \nu \bm{\xi}_{k}
\end{equation}
where $\bm{\xi}_{k} \sim \mathcal{N}(0,I)$, $\eta$ is the step size and $k \in [0,K-1]$. We label the resultant particles as $\rvx_K$. The $\rvx_K$ are drawn from the distribution $\Tilde{q}_\tau(\rvx_K)$ given by
\begin{equation}\label{eq:transition_kernel}
    \Tilde{q}_\tau(\rvx_K) = \int q'(\rvx')M_{\theta_{\tau}}(\rvx_K|\rvx')d\rvx
\end{equation}
where $M_{\theta_{\tau}}(\rvx_K|\rvx)$ is the transition kernel of simulating Eq.~\ref{eq:sde_euler_maruyama} with parameters $\theta_{\tau}$. We optimize $r_{\theta_\tau}(\rvx)$ for the $\tau$ training iteration using the Bregman divergence 
\begin{equation}\label{eq:bd_t_n}
    \mathcal{L}(\theta_\tau) = \mathbb{\widehat{E}}_p[\partial g(r_{\theta_
    \tau}(\rvx))r(\rvx) - g(r_{\theta_\tau}(\rvx)) ] - \mathbb{\widehat{E}}_{\Tilde{q}_\tau} [ \partial g(r_{\theta_\tau}(\rvx))]
\end{equation}
where expectation over $p(\rvx)$ can be estimated using samples from the dataset. We update the parameters for the next training step $\tau+1$ as $\theta_{\tau+1} \leftarrow \theta_\tau - \alpha \nabla_\theta \mathcal{L}(\theta_\tau)$ with learning rate $\alpha$. The process repeats until convergence. We provide pseudocode in Algorithm \ref{alg:training}, and an illustration of the training setup in Fig.~\ref{fig:main}. %
We illustrate our training process in Fig. \ref{fig:toy_example}, where we utilize the toy example from Fig. \ref{fig:stale_no_works} which the stale estimator $r_\theta = q'/p$ failed previously. We can see from the top row that as $\tau$ increases, the orange samples from $\Tilde{q}_\tau$, which are obtained by flowing the blue samples from $q'$ for fixed $K$ steps, progressively approach the target distribution. In the bottom row, we plot the progression of the mean of $\Tilde{q}_\tau$ as training progresses, which shows the iterative approach towards the target distribution. Together with the experiments for high-dimensional image synthesis in Sec.~\ref{sec:exp}, these results empirically justify that \modelname{} indeed overcomes the density chasm problem, allowing the stale estimator $r_\theta(\rvx)$ to be used for sample generation.

\begin{figure}
    \centering
    \includegraphics[width=\linewidth, trim={1cm 1cm 1cm 1.5cm},clip]{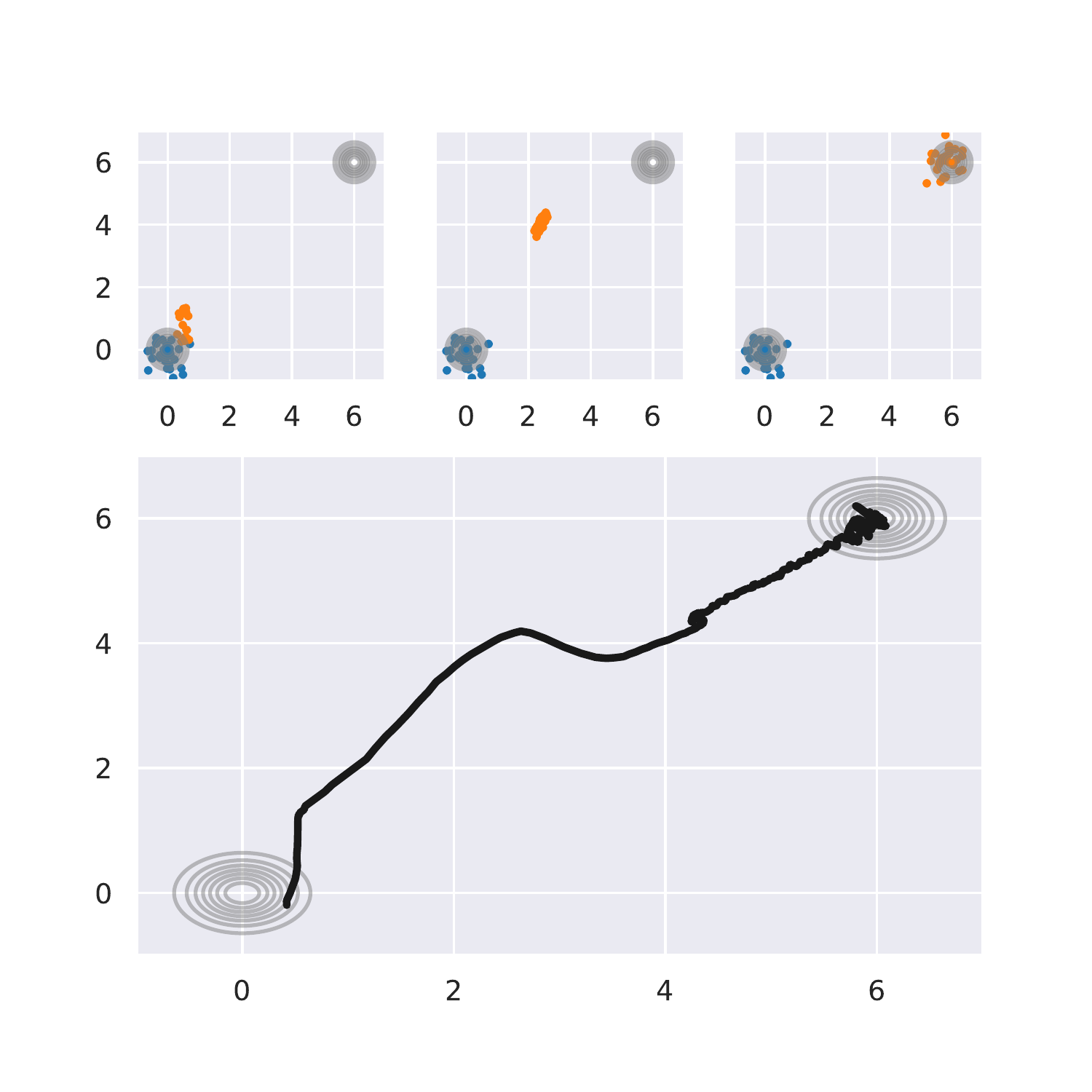}
    \caption{Plot of $\Tilde{q}_\tau$ as training progresses in the toy example of Fig. \ref{fig:stale_no_works}. The total number of training steps is $T=1000$, such that $\tau \in [0, T-1]$. In the top row, blue particles are samples from $q'$ while orange particles are samples from $\Tilde{q}_\tau$ at specific training steps. Left: $\tau = 10$, center: $\tau=40$, right: $\tau=1000$. In the bottom row, we plot the trajectory of the mean of $\Tilde{q}_\tau$ as training progresses.}
    \label{fig:toy_example}
\end{figure}

\subsection{Convergence Distribution of \modelname{}}\label{sec:conv}
In this section, we analyze the convergence distribution after training has converged. With $T$ training steps and $\tau \in [0, T-1]$, let us consider what happens during the final training iteration $T-1$. The model has parameters $\theta_{T-1}$ from the gradient update of the previous iteration. Following Algorithm \ref{alg:training}, we first sample $\rvx_K$ from the distribution $\Tilde{q}_{T-1}(\rvx_K) = \int q'(\rvx')M_{\theta_{T-1}}(\rvx_K|\rvx')d\rvx'$ by running $K$ steps of Eq. \ref{eq:sde_euler_maruyama}. With samples $\rvx_K \sim \Tilde{q}_{T-1}(\rvx)$ and $\rvx_p \sim p(\rvx)$, we optimize the BD and take the gradient step $\theta_{T} \leftarrow \theta_{T-1} - \alpha \nabla_\theta \mathcal{L}$. Assuming convergence, this means by definition that the loss gradient of the final step $\nabla_\theta \mathcal{L} = 0$. Thus, with no further gradient updates possible, our model is parameterized by $\theta_{T}$ after convergence. We can now analyze what happens when sampling from our trained model with parameters $\theta_{T}$.

In our framework, we sample from our trained model by flowing for a total of $K + \kappa$ steps. This sampling can be viewed as two stages ($K$ ``bridging'' steps, followed by $\kappa$ refinement steps). To elaborate, in the first stage, we sample as before by running Eq. \ref{eq:sde_euler_maruyama} for $K$ steps using our trained model $r_{\theta_{T}}$. The samples from stage 1 will be drawn from  $\Tilde{q}_{T}(\rvx_K) = \int q'(\rvx')M_{\theta_{T}}(\rvx_K|\rvx')d\rvx'$. Based on empirical evidence from our toy example in Fig.~\ref{fig:toy_example} and image experiments in Sec.~\ref{sec:image_gen}, $\Tilde{q}_{T}(\rvx_K)$ is closer to the data distribution $p(\rvx)$. Interestingly, we can introduce a second refinement stage. Recall that convergence in the Bregman divergence, $\nabla_\theta \mathcal{L} = 0$, implies that our density ratio estimator has converged to the value $r_{\theta_{T}}(\rvx) = \Tilde{q}_{T-1} / p$. Since after the first stage, we have samples from $\Tilde{q}_{T}(\rvx)$, and an estimator $r_{\theta_{T}}(\rvx) = \Tilde{q}_{T-1} / p \approx \Tilde{q}_{T} / p$ (assuming the change in $\theta$ in the final training step is small), we can perform sample refinement in exactly the formalism of DG$f$low by running the same Eq. \ref{eq:sde_euler_maruyama} for a further $\kappa$ steps. %
We provide pseudocode for sampling in Algorithm~\ref{alg:sampling} and justify the two-stage sampling scheme with ablations in Sec.~\ref{sec:ablations}. 

\begin{algorithm}
\caption{Training}\label{alg:training}
\begin{algorithmic}

\Repeat
    \State Sample $\rvx_p \sim p(\rvx), \rvx_{0} \sim q'(\rvx)$ 
    \For {$j \gets 1,K $}
        \State Obtain $\rvx_K$ from $\rvx_0$ by simulating Eq.~\ref{eq:sde_euler_maruyama}.
    \EndFor
    
    \State Update $\theta$ according to
    \begin{equation*}
    \begin{multlined}[c]
         \nabla_\theta [ g'(r_\theta(\rvx_p))r(\rvx_p) - g(r_\theta(\rvx_p)) - g'(r_\theta(\rvx_K))]
    \end{multlined}
    \end{equation*}
\Until{converged}
\end{algorithmic}
\end{algorithm}
\begin{algorithm}
\caption{Sampling}\label{alg:sampling}
\begin{algorithmic}
\State Sample $\rvx_{0} \sim q'(\rvx)$
\For {$j \gets 1,K+\kappa $}
    \State Obtain $\rvx_{K+\kappa}$ from $\rvx_0$ by simulating Eq.~\ref{eq:sde_euler_maruyama}.
\EndFor 
\State \Return $\rvx_{K+\kappa}$
\end{algorithmic}
\end{algorithm}

\begin{figure*}[h]
     \centering
     \begin{subfigure}[b]{0.32\textwidth}
         \centering
         \includegraphics[width=\textwidth]{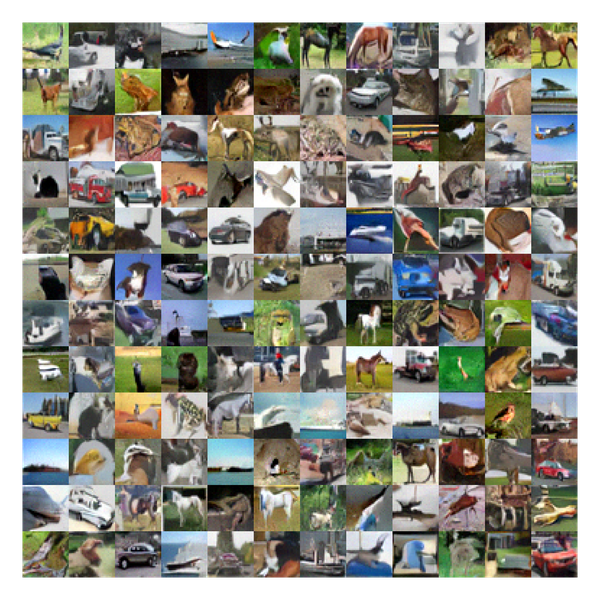}
         \label{fig:cifar-lsif-pearson}
     \end{subfigure}
     \hfill
     \begin{subfigure}[b]{0.32\textwidth}
         \centering
         \includegraphics[width=\textwidth]{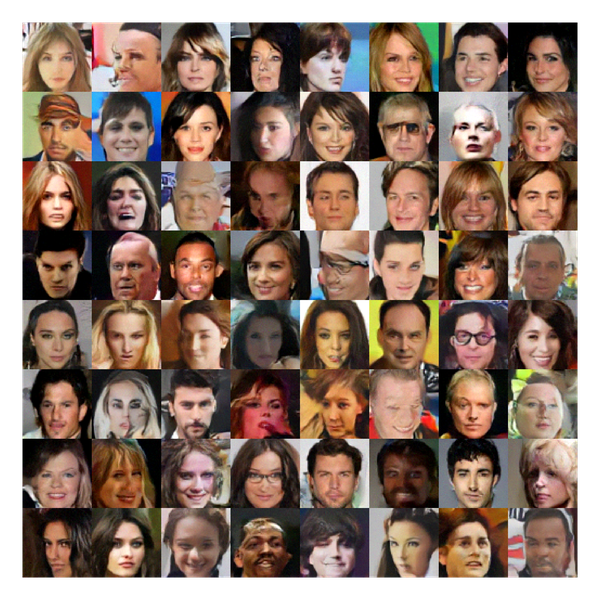}
         \label{fig:celeba-lsif-pearson}
     \end{subfigure}
     \hfill
     \begin{subfigure}[b]{0.32\textwidth}
         \centering
         \includegraphics[width=\textwidth]{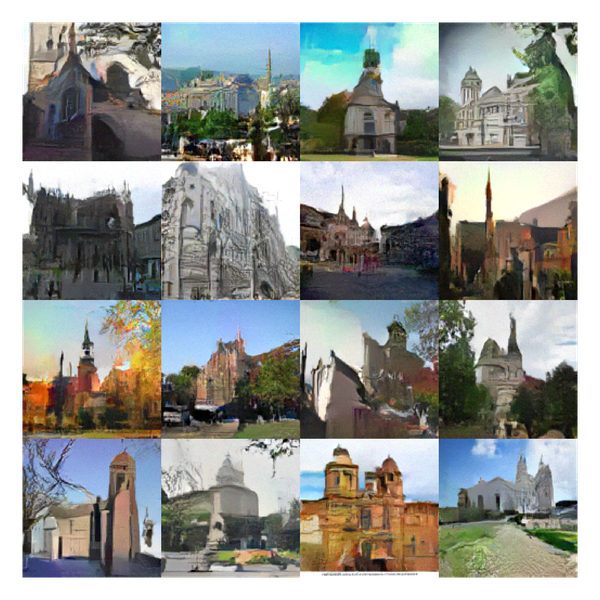}
         \label{fig:lsun-lsif-pearson}
     \end{subfigure}
        \caption{Samples from \modelname-DDP on CIFAR10 $32^2$, CelebA $64^2$ and LSUN Church $128^2$ using LSIF-$\chi^2$. More results using various BD objectives and $f$-divergences can be found in Appendix. }
        \label{fig:main_result}
\end{figure*}

\section{Related Works}\label{sec:related_works}
Gradient flows are a general framework for constructing the steepest descent curve of a given functional, and consequently have been used in optimizing a variety of distance metrics, ranging from the $f$-divergence \cite{gao2019deep,gao2022deep,ansari2020refining,fan2021variational}, maximum mean discrepancy \cite{arbel2019maximum,mroueh2021convergence}, Sobolev distance \cite{mroueh2019sobolev} and related forms of the Wasserstein distance \cite{liutkus2019sliced}.  

Two well-known methodologies to simulate Wasserstein gradient flows are population density approaches such as the JKO scheme \cite{jordan1998variational} and particle-based approaches. The former requires computation of the 2-Wasserstein distance and a free energy functional, which are generally intractable. As such, several works have proposed approximations leveraging Brenier's theorem ~\cite{mokrov2021large}, Fenchel conjugate~\cite{fan2021variational} and neural forward and backward schemes~\cite{altekruger2023neural}. 

On the other hand, particle approaches involve simulating the SDE Eq.~\ref{eq:sde}. DG$f$low leverages GAN discriminators as a stale density ratio estimator for sample refinement. Our method is inspired by the stale formulation of DG$f$low, but we extend it to the general case of sample generation. In fact, DG$f$low is a specific instance of \modelname{} where we fix the prior $q'$ to be the implicit distribution defined by the GAN generator. Euler Particle Transport (EPT) \cite{gao2022deep} similarly trains a deep density ratio estimator using the Bregman divergence. However, EPT was only shown to be effective in the latent space and thus necessitated the training of an additional generator network, whereas our method is able to scale directly in the data space. Interestingly, recent work \cite{franceschi2024unifying} has shown that particle approaches like \modelname{}, which train networks to estimate certain functionals (e.g., density ratio or score~\cite{song2019generative}), can be unified with methods that explicitly train generator networks such as GANs~\cite{goodfellow2014generative}. This view supports the methodology of \modelname{}, where the density ratio estimator is trained as a density discriminator but used as a generator during inference.

Another related line of work are energy-based models (EBMs), which adopt Langevin dynamics to sample from an unnormalized distribution \cite{lecun2006tutorial,xie2016theory}. The training scheme of \modelname{} bears similarity with that of Short-run EBM \cite{nijkamp2019learning}, where a sample is drawn by restarting the Langevin chain and running $K$ gradient steps. However, \modelname{} is fundamentally distinct from EBMs; the former is trained with density ratio estimation, while the latter is trained with maximum likelihood. Further, it can be shown that Langevin dynamics is in fact a special case of the KL gradient flow \cite{jordan1998variational,liu2019understanding}.

\section{Experiments}\label{sec:exp}
The goal of our experiments is to compare FDRL against gradient flow baselines. However, for completeness, we also include results from other modeling approaches, including EBMs, GANs, and Diffusion models.

\subsection{Setup}
We test \modelname{} with unconditional generation on 32$\times$32 CIFAR10, 64$\times$64 CelebA and 128$\times$128 LSUN Church datasets, class-conditional generation on CIFAR10 and image-to-image translation on 64$\times$64 Cat2dog dataset. We label experiments with the DDP as \modelname{}-DDP, and ablation experiments with a uniform prior as \modelname{}-UP. We use modified ResNet architectures for all experiments in this study. See Appendix \ref{app:exp_details} for more details. 

\subsection{Unconditional Image Generation}\label{sec:image_gen}

\begin{table}
\centering
\caption{FID scores for 32$\times$32 CIFAR10 and 64$\times$64 CelebA. Lower is better.}
\label{tab:fid}
\resizebox{\linewidth}{!}{%
\begin{tblr}{
  width = \linewidth,
  colspec = {Q[640]Q[150]Q[120]},
  cell{1}{2} = {c},
  cell{3}{2} = {c},
  cell{3}{3} = {c},
  cell{4}{2} = {c},
  cell{4}{3} = {c},
  cell{5}{2} = {c},
  cell{5}{3} = {c},
  cell{6}{2} = {c},
  cell{6}{3} = {c},
  cell{7}{2} = {c},
  cell{7}{3} = {c},
  cell{9}{2} = {c},
  cell{9}{3} = {c},
  cell{10}{2} = {c},
  cell{10}{3} = {c},
  cell{12}{2} = {c},
  cell{12}{3} = {c},
  cell{13}{2} = {c},
  cell{13}{3} = {c},
  cell{15}{2} = {c},
  cell{15}{3} = {c},
  cell{16}{2} = {c},
  cell{16}{3} = {c},
  cell{17}{2} = {c},
  cell{17}{3} = {c},
  cell{19}{2} = {c},
  cell{19}{3} = {c},
  cell{20}{2} = {c},
  cell{20}{3} = {c},
  cell{21}{2} = {c},
  cell{21}{3} = {c},
  cell{22}{2} = {c},
  cell{22}{3} = {c},
  cell{23}{2} = {c},
  cell{23}{3} = {c},
  cell{24}{2} = {c},
  cell{24}{3} = {c},
  hline{1-2,8,11,14,18,25} = {-}{},
  rowsep=-1pt
}
\textbf{Method} & \textbf{CIFAR10} & \textbf{CelebA}\\
\textbf{\textit{EBM-based}} &  & \\
IGEBM~\cite{du2019implicit} & 40.58 & -\\
SR-EBM~\cite{nijkamp2019learning} & 44.50 & 23.02\\
LEBM~\cite{pang2020learning} & 70.15 & 37.87\\
CoopNets~\cite{xie2018cooperative} & 33.61 & -\\
VAEBM~\cite{xiao2020vaebm} & 12.19 & 5.31\\
\textbf{\textit{Diffusion-based}} &  & \\
NCSN~\cite{song2019generative} & 25.32 & 26.89\\
DDPM~\cite{ho2020denoising} & 3.17 & -\\
\textbf{GAN-based} &  & \\
WGAN-GP~\cite{gulrajani2017improved} & 55.80 & 30.00\\
SNGAN~\cite{miyato2018spectral} & 21.70 & -\\
\textbf{\textit{Gradient flow-based}} &  & \\
EPT~\cite{gao2019deep} & 24.60 & -\\
JKO-Flow~\cite{fan2021variational} & 23.70 & -\\
$\ell$-SWF~\cite{du2023nonparametric} & 59.70 & 38.30\\
\textbf{\textit{Ours}} &  & \\
FDRL-DDP (LSIF-$\chi^2$) & 22.28 & 17.77\\
FDRL-DDP (LR-KL) & 22.61 & 18.09\\
FDRL-DDP (LR-JS) & 22.80 & -\\
FDRL-DDP (LR-logD) & 22.82 & -\\
FDRL-UP (LSIF-$\chi^2$ & 30.23 & -\\
FDRL-UP (LR-KL) & 31.99 & -
\end{tblr}
}
\end{table}

In Fig. \ref{fig:main_result}, we show uncurated samples of \modelname{}-DDP on different combinations of $g$ and $f$-divergences, where the combinations are chosen due to numerical compatibility (see Appendix \ref{sec:app_pairing}). Visually, our model is able to produce high-quality samples on a variety of datasets up to resolutions of $128\times128$, surpassing existing gradient flow techniques~\cite{gao2022deep, fan2021variational}. Samples with other $f$-divergences can be found in Appendix \ref{app:ddp_samples}. In Table \ref{tab:fid}, we show the FID scores of \modelname-DDP in comparison with relevant baselines utilizing different generative approaches. 

We emphasize that our experiments are geared towards comparisons with gradient flow baselines, in which \modelname{}-DDP outperforms both EPT and JKO-Flow, as well as the non-parametric method $\ell$-SWF. Our method outperforms all baseline EBMs except VAEBM, which combines a VAE and EBM for improved performance. We also outperform WGAN and the score-based NCSN, while matching SNGAN's performance. Comparing method approaches, modern diffusion models (DDPM) outperform EBM and gradient flow methods, including \modelname{}. Interesting future work could investigate reasons for this gap, which may lead to improvements for EBMs and gradient flows.

To intuit the flow process, we show how samples evolve with the LSUN Church dataset in Fig. \ref{fig:gradient_flow_process} of the appendix, where the prior contains low frequency features such as the rough silhouette of the building, and the flow process generates the higher frequency details to create a realistic image.

We also visualize samples obtained by interpolating in the prior space for CelebA in Fig. \ref{fig:latent_space_interpolation} of the appendix. The model smoothly interpolates in the latent space despite using the DDP, indicating a semantically relevant latent representation that is characteristic of a valid generative model. Finally, to verify that our model has not merely memorized the dataset, we show nearest neighbor samples of the generated images in the training set in Appendix \ref{app:nn}. The samples produced by \modelname-DDP are distinct from their closest training samples,
telling us that \modelname-DDP is capable of generating new and diverse samples beyond the training data.

\subsection{Ablations}\label{sec:ablations}

\begin{figure}
    \centering
    \includegraphics[width=\linewidth, trim={0cm 0.5cm 0cm 0cm},clip]{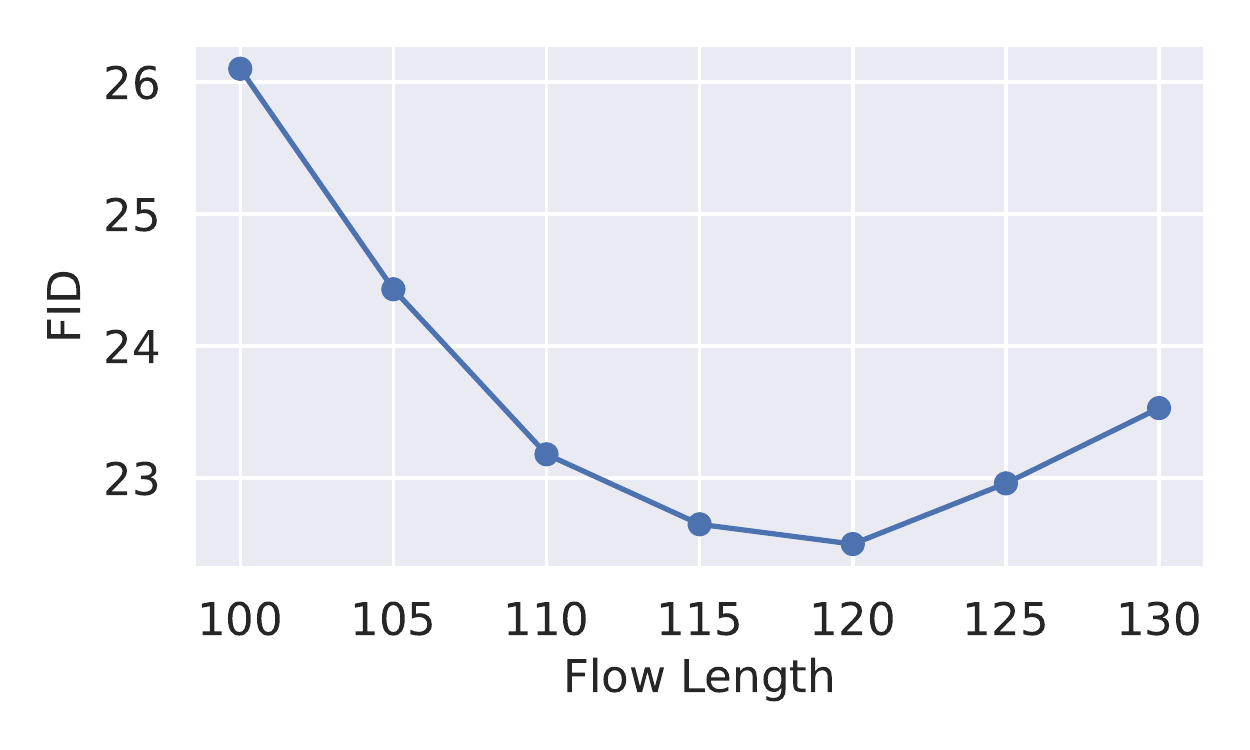}
    \caption{FID as a function of the total flow length on CIFAR10 for LSIF-$\chi^2$ DDP when sampling from a model trained with $K=100$.}
    \label{fig:k_vs_fid}
\end{figure}
\paragraph{Data-Dependent Prior vs Uniform Prior}
We motivate the use of the data-dependent prior in conjunction with \modelname{} by conducting ablations with $q'$ being a uniform prior (UP), $\rvx_{0} \sim U[-1,1]$. We train UP runs for 20$\%$ more training steps than DDP runs, but keep all other hyperparameters unchanged for control. We include qualitative samples of \modelname-UP in Fig. \ref{fig:up_samples} in the appendix. Table \ref{tab:fid} reports the quantitative scores. Visually, \modelname-UP produces diverse and appealing samples even with a uniform prior. However, the quantitative scores are significantly poorer than using the DDP, which validates our hypothesis in Sec. \ref{sec:ddp}.

\paragraph{Length of Flow vs Image Quality}
Fig.~\ref{fig:k_vs_fid} illustrates the relationship between total flow length and FID. Although the model is trained with $K=100$,
sampling with $K$ steps at testing does not produce the best image quality, as sampling with an additional $\kappa=20$ steps yields the best FID score. This supports the two-stage sampling approach discussed in Sec.~\ref{sec:conv}. First, $K=100$ steps of the flow are run to sample images from $\Tilde{q}_T(\rvx_K)$, followed by $\kappa=20$ additional steps of the same Eq. \ref{eq:sde_euler_maruyama} using the estimator $r{\theta_T} \approx \Tilde{q}_T/p$ for sample refinement, as done in DG$f$low.

\subsection{Conditional Image Generation}
\begin{figure}
    \centering
    \includegraphics[width=\linewidth]{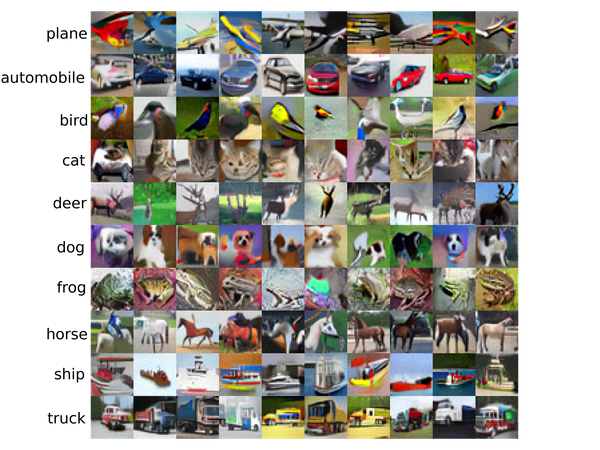}
    \caption{Class-conditional samples obtained by composing an unconditional \modelname{} with a pretrained robust classifier.}
    \label{fig:class_conditional}
\end{figure}
We can sample from target densities different from what \modelname{} was trained on by simply composing density ratios. Consider a multiclass classifier which classifies a given image into one of $N$ classes. We show in Appendix \ref{app:classifier} that we can express such classifiers as a density ratio $p(y=n|\rvx) = N^{-1} p(\rvx|y=n)/p(\rvx)$. Thus, we can obtain a conditional density ratio estimator $r_\theta(\rvx|y=n)=q_{T}(\rvx)/p(\rvx|y=n)$ by composing our unconditional estimator $r_{\theta}(\rvx)$ with the classifier output (see Appendix \ref{app:classifier}):
\begin{equation}
    r_\theta(\rvx|y=n) = \frac{1}{N} r_\theta(\rvx)p(y=n|\rvx)^{-1}.
\end{equation}
When $r_\theta(\rvx|y=n)$ is used in the flow in Eq. \ref{eq:sde_euler_maruyama}, we perform class-conditional generation. This is conceptually similar to the idea proposed in \cite{song2020score, dhariwal2021diffusion}, where an unconditional score model $\nabla_\rvx \log p_t(\rvx(t))$ is composed with a time-dependent classifier $\nabla_\rvx \log p_t(y|\rvx(t))$ to form a class-conditional model. However, whereas \cite{song2020score} requires separately training a time-dependent classifier, our formulation allows the use pretrained classifiers off-the-shelf, with \emph{no further retraining}. Inspired by earlier work on image synthesis with robust classifiers~\cite{santurkar2019image}, we found that using a pretrained adversarially-robust classifier was necessary in obtaining useful gradients for generation. We show our results in Fig. \ref{fig:class_conditional}, where each row represents conditional samples of each class in the CIFAR10 dataset.

\subsection{Unpaired Image-to-image 
Translation}\label{sec:im2im_translation}
\begin{figure}
    \centering
    \includegraphics[width=0.87\linewidth]{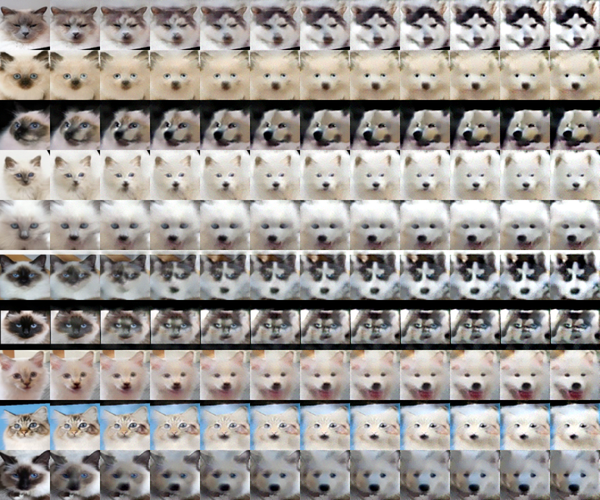}
    \caption{Image-to-image translation process from cat to dog images using \modelname{}.}
    \label{fig:im2im_cat2dog}
\end{figure}

\modelname{} can also be seamlessly applied to unpaired image-to-image-translation (I2I). We simply fix the prior distribution $q'(\rvx)$ to a source image domain and $p(\rvx)$ to a target domain. We then train the model in exactly the same manner as unconditional generation using Algorithm \ref{alg:training}. 

The I2I model is tested on the Cat2dog dataset~\cite{lee2018diverse}. From Fig. \ref{fig:im2im_cat2dog}, \modelname{} is able to smoothly translate images of cats to dogs while preserving relevant features, such as the background colors, pose and facial tones. For instance, a cat with light fur is translated to a dog with light fur. Quantitatively, the I2I baseline CycleGAN~\cite{zhu2017unpaired} achieves better FID scores than \modelname{} (Appendix Table \ref{tab:im2im-fid}). However, like many I2I methods~\cite{lee2018diverse, choi2020stargan,zhao2021unpaired, nie2021controllable}, CycleGAN relies on specific inductive biases, such as dual generators and discriminators, and cycle-consistency loss. Future work could explore incorporating these biases into \modelname{} to improve translation performance.

\section{Conclusion}
We propose \modelname{}, a method that enables the stale approximation of the gradient flow to be applied to generative modeling of high-dimensional images. Beyond unconditional image generation, the \modelname{} framework can be seamlessly adapted to other key applications, including class-conditional generation and unpaired image-to-image translation. Future work could focus on theoretically characterizing the exact distribution induced by the stale SDE of Eq.~\ref{eq:sde_stale} and further improving our understanding of its convergence properties. 

\section{Acknowledgements}
This research is supported by the National Research Foundation Singapore and DSO National Laboratories under the AI Singapore Programme (AISG Award No: AISG2-RP-2020-016).

{\small
\bibliographystyle{ieee_fullname}
\bibliography{refs}
}

\clearpage
\newpage

\appendix
\input{appendix.tex}

\end{document}

%% file: math_commands.tex
\usepackage{amsmath,amsfonts,bm}

\def\eqref#1{equation~\ref{#1}}

\def\1{\bm{1}}

\def\rvw{{\mathbf{w}}}
\def\rvx{{\mathbf{x}}}

\DeclareMathAlphabet{\mathsfit}{\encodingdefault}{\sfdefault}{m}{sl}
\SetMathAlphabet{\mathsfit}{bold}{\encodingdefault}{\sfdefault}{bx}{n}

\DeclareMathOperator*{\argmax}{arg\,max}
\DeclareMathOperator*{\argmin}{arg\,min}

%% file: appendix.tex
\begin{figure*}
    \centering
    \includegraphics[width=\textwidth]{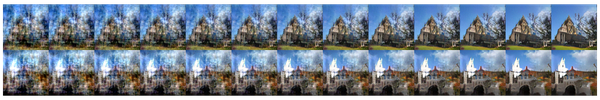}
    \caption{Illustration of the flow process for LSUN Church, starting from a sample from the data-dependent prior on the leftmost column.}
    \label{fig:gradient_flow_process}
\end{figure*}
\begin{figure*}
    \centering
    \includegraphics[width=\textwidth]{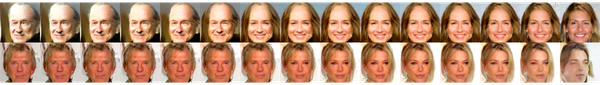}
    \caption{Interpolation results between leftmost and rightmost samples with CelebA.}
    \label{fig:latent_space_interpolation}
\end{figure*}

\section{Proofs}
\mainlemma*
\begin{proof}
The proof draws inspiration from the Langevin equation. Therefore, let us first consider the standard overdamped Langevin SDE:
\begin{equation}\label{eq:langevin}
    d\rvx_t = -\nabla f(\rvx_t)dt + \sqrt{2}d\rvw_t
\end{equation}
where $f$ is Lipschitz. The associated Fokker-Planck equation is given by~\cite{risken1985fokker}
\begin{equation}\label{eq:fpe}
    \partial_t \rho_t = \nabla \cdot (\rho \nabla f) + \Delta \rho.
\end{equation}
where $\Delta$ is the Laplace operator. It is straightforward to show that the stationary distribution of Eq.~\ref{eq:langevin} is given by
\begin{equation}\label{eq:stationary}
    \rho_\infty(\rvx) = \frac{e^{-f(\rvx)}}{\int e^{-f(\rvx)}d\rvx} = \frac{e^{-f(\rvx)}}{Z}.
\end{equation}
We simply verify by substituting Eq.~\ref{eq:stationary} into Eq.~\ref{eq:fpe}:
\begin{align*}
\partial_t \rho_\infty & =  \nabla \cdot \left(\frac{e^{-f(\rvx)}}{Z}  \nabla f \right) + \nabla \cdot \nabla \left(\frac{e^{-f(\rvx)}}{Z}\right) \\ & =
\nabla \cdot \left(\frac{e^{-f(\rvx)}}{Z} \nabla f\right) - \nabla \cdot \left(\frac{e^{-f(\rvx)}}{Z} \nabla f\right) \\ &=0.
\end{align*}
Now, let us return to the stale estimate SDE of Eq.~\ref{eq:sde_stale} with $\gamma=1$ and $q(\rvx)\sim U[a,b]=\frac{1}{b-a}=C$:
\begin{equation}\label{eq:app_sde_stale}
    d\rvx_t = -\nabla f'(C/p(\rvx_t))dt + \sqrt{2}d\rvw_t.
\end{equation}
Since $f$ is convex and twice-differentiable, $f'$ is a monotonically increasing function. Using what we showed above, the stationary distribution of Eq.~\ref{eq:app_sde_stale} is thus
\begin{equation}
    \rho_\infty(\rvx) = \frac{e^{-f'(C/p(\rvx))}}{\int e^{-f(C/p(\rvx))}d\rvx} = \frac{e^{-f'(C/p(\rvx))}}{Z}.
\end{equation}
Then we can show that
\begin{align*}
& \argmax_\rvx \rho_\infty(\rvx) \\ 
&= \argmax_\rvx \frac{e^{-f'(C/p(\rvx))}}{Z} \\ 
&= \argmax_\rvx e^{-f'(C/p(\rvx))} \\
&= \argmax_\rvx -f'(C/p(\rvx)) && \text{($e$ is monotonically increasing)} \\
&= \argmin_\rvx f'(C/p(\rvx)) \\
&= \argmin_\rvx C/p(\rvx) && \text{($f'$ is monotonically increasing)} \\
&= \argmax_\rvx p(\rvx) && \text{($p(\rvx) > 0$)}.
\end{align*}
\end{proof}

\section{Toy Datasets}
\begin{figure*}[h]
    \centering
    \includegraphics[width=\textwidth, trim={6.5cm 0 5cm 0},clip]{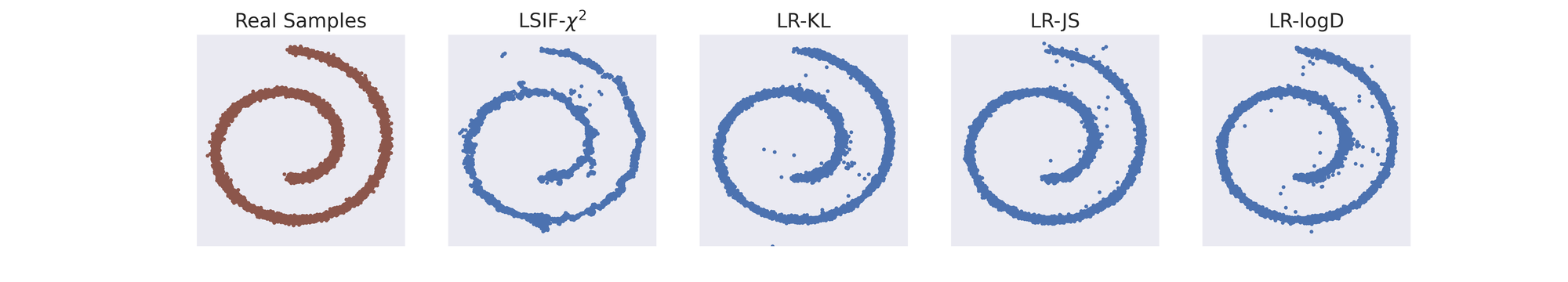}
    \caption{Comparison of different \modelname{} pairings of Bregman and $f$-divergences on the 2DSwissroll dataset.}
    \label{fig:swissroll}
\end{figure*}
To affirm that samples generated by \modelname{} indeed converge to the target distribution, we train \modelname{} on the synthetic 2DSwissroll dataset. The density ratio estimator is parameterized by a simple feedforward multilayer perceptron. We train the model to flow samples from the prior $q'(\rvx)=\mathcal{N}(\mathbf{0},\mathbb{I})$ to the target distribution, which we sample from the \texttt{make\_swiss\_roll} function in \texttt{scikit-learn}. We plot the results in Fig. \ref{fig:swissroll}, from which we can see that the model indeed recovers $p(\rvx)$ successfully for all combinations of $f$ and $g$. 

\section{Bregman Divergence and $f$-divergence Pairing}\label{sec:app_pairing}

We elucidate the full forms of the Bregman divergences here. The LSIF objective $g(y)=\frac{1}{2}(y-1)^2$ is given as
\begin{equation}\label{eq:lsif}
    \mathcal{L}_{LSIF}(\theta) = \frac{1}{2}\mathbb{\widehat{E}}_p[r_\theta(\rvx)]^2 - \mathbb{\widehat{E}}_{q_\tau} [r_\theta(\rvx)],
\end{equation}
while the LR objective $g(t) =y\log y - (1+y)\log(1+y) $ is given as
\begin{equation}\label{eq:lr}
    \mathcal{L}_{LR}(\theta) = -\mathbb{\widehat{E}}_p \left[ \log \frac{1}{1+r_\theta(\rvx)} \right] - \mathbb{\widehat{E}}_{q_\tau} \left[ \log \frac{r_\theta(\rvx)}{1+r_\theta(\rvx)} \right].
\end{equation}

When computing the LR objective, we find that we run into numerical stability issues when letting $r_\theta(\rvx)$ be the output of an unconstrained neural network and subsequently taking the required logarithms in Eq. \ref{eq:lr}. To circumvent this issue, we let $r_\theta(\rvx)$ be expressed as the exponential of the neural network's output, i.e., the output of the neural network is $\log r_\theta(\rvx)$. This formulation naturally lends itself to the  flow of the KL, JS and logD divergences, whose first derivatives $f'$ that is required in Eq. \ref{eq:sde_euler_maruyama} are also logarithmic functions of $r_\theta(\rvx)$, as seen from Table. \ref{tab:f-div}. We can thus utilize numerically stable routines in existing deep learning frameworks, avoiding the need for potentially unstable operations like exponentiations (see Appendix \ref{sec:app_stable} for details). As such, we pair LR with the aforementioned divergences and abbreviate the combinations as LR-KL, LR-JS, LR-logD. We did not run into such stability issues for the LSIF objective (Eq. \ref{eq:lsif}) as the model learns to automatically output a non-negative scalar over the course of training, hence for LSIF we allow the neural network to estimate $r_\theta(\rvx)$ directly and pair it with the Pearson-$\chi^2$ divergence. We abbreviate this pairing as LSIF-$\chi^2$.

\section{Stable Computation of LR and $f$-divergences}\label{sec:app_stable}
As mentioned in Sec. \ref{sec:app_pairing}, computing the logarithm of unconstrained neural networks leads to instabilities in the training process. This is a problem when computing the LR objective in Eq. \ref{eq:lr} and the various first derivatives of $f$-divergences. We can circumvent this by letting $r_\theta(\rvx)$ be expressed as the exponential of the neural network and use existing stable numerical routines to avoid intermediate computations that lead to the instabilities (for example, computing logarithms and exponentials directly). Let us express the neural network output as $NN_\theta(\rvx) \triangleq \log r_\theta(\rvx)$. The LR objective can then be rewritten as

\begin{align}
    \mathcal{L}_{LR}(\theta) &= -\mathbb{\widehat{E}}_p \left[ \log \frac{1}{1+r_\theta(\rvx)} \right] - \mathbb{\widehat{E}}_{q_\tau} \left[ \log \frac{r_\theta(\rvx)}{1+r_\theta(x)} \right] \\
    &= -\mathbb{\widehat{E}}_p \left[ \texttt{LS}(-NN_\theta(x))\right] - \mathbb{\widehat{E}}_{q_\tau} \left[ \texttt{LS}(NN_\theta(x)) \right]
\end{align}
where $\texttt{LS}(x)=\log \frac{1}{1+\exp (-x)}$, the log-sigmoid function, which has stable implementations in modern deep learning libraries.

Similarly for the $f$-divergences whose first derivatives involve logarithms, we can calculate them stably as
\begin{align}
    f'_{KL}(r_\theta(\rvx)) &= \log r_\theta(\rvx) + 1 = NN_\theta(\rvx) + 1\\
    f'_{JS}(r_\theta(\rvx)) &= \log \frac{2r_\theta(\rvx)}{1+r_\theta(\rvx)} = \log 2 + \texttt{LS}(NN_\theta(\rvx)) \\
    f'_{logD}(r_\theta(\rvx)) &= \log (r_\theta(\rvx)+1) + 1 = -\texttt{LS}(-NN_\theta(\rvx)) + 1.
\end{align}

\begin{table*}[]
\centering
\caption{$f$-divergences and their first derivatives $f' $.}
\label{tab:f-div}
\begin{tabular}{lll}
\hline
$f$-divergence   & \multicolumn{1}{c}{$f$}               & \multicolumn{1}{c}{$f'$} \\ \hline
Pearson-$\chi^2$ & $(r-1)^2$                             & $2(r-1)$                 \\
KL               & $r \log r$                            & $\log r + 1$             \\
JS               & $r \log r - (r+1) \log \frac{r+1}{2}$ & $\log \frac{2r}{r+1}$    \\
log D            & $(r+1) \log (r+1) - 2 \log 2 $        & $\log (r+1) + 1$         \\ \hline
\end{tabular}
\end{table*}

\section{Classifiers are Density Ratio Estimators}\label{app:classifier}
To perform class-conditional generation in the \modelname{} framework, we would like to estimate the density ratio of a certain class over the data distribution: $p(\rvx|y=n)/p(\rvx)$. With Bayes rule, we can write this as 
\begin{align}
    \frac{p(\rvx|y=n)}{p(\rvx)} &= \frac{p(y=n|\rvx)p(\rvx)/p(y=n)}{p(\rvx)} \\
    &= \frac{p(y=n|\rvx)}{p(y=n)}.
\end{align}
The denominator term $p(y=n)$ can be viewed as a constant, e.g. assume the $N$ classes are equally distributed, then $p(y=n) = 1/N$. Therefore, we have that the class probability given by the softmax output of a classifier is actually a density ratio:
\begin{equation}
    Np(y=n|\rvx) = \frac{p(\rvx|y=n)}{p(\rvx)}.
\end{equation}
We can use this equation to convert an unconditional \modelname{} to a class-conditional generator. Recall the stale approximation of the gradient flow:
\begin{equation}\label{eq:recall_sde}
    d\rvx_t = -\nabla_\rvx f'(r_\theta(\rvx_t))dt + \sqrt{2\gamma}d\rvw_t
\end{equation}
Consider the case where we have a trained unconditional \modelname{} estimator $r_{\theta}$. We can multiply the inverse of the classifier output with $r_\theta(\rvx_t) = \Tilde{q}_{\tau}(\rvx_t)/p(\rvx_t)$ to get a density ratio between $\Tilde{q}_{\tau}(\rvx_t)$ and the conditional data distribution $p(\rvx_t|y=n)$:
\begin{align*}
    r_\theta(\rvx_t) p(y=n|\rvx_t)^{-1} &= \frac{\Tilde{q}_\tau(\rvx_t)}{p(\rvx_t)} \frac{Np(\rvx_t)}{p(\rvx_t|y=n)} \\ &= N\frac{\Tilde{q}_\tau(\rvx_t)}{p(\rvx_t|y=n)}.
\end{align*}
That is, we took our unconditional model and converted it to a conditional generative model by composing it with a pretrained classifier. To get the correct class-conditional density ratio that can be used for generation, we should therefore compute 
\begin{equation}
    r_\theta(\rvx_t|y=n) = \frac{1}{N} r_\theta(\rvx_t)p(y=n|\rvx_t)^{-1}
\end{equation}
and use this conditional density ratio estimator in our sampling method of Algorithm Eq.~\ref{alg:sampling}.

\section{Experimental Details}\label{app:exp_details}
\begin{table*}[]
\caption{Network structures for the density ratio estimator $r_\theta(\rvx)$.}
\label{tab:archi}
\begin{minipage}{.24\textwidth}
\begin{tabular}{c}
CIFAR10                \\ \hline
3$\times$3 Conv2d, 128 \\
3 $\times$ ResBlock 128           \\
ResBlock Down 256      \\
2 $\times$ ResBlock 256           \\
ResBlock Down 256      \\
2 $\times$ ResBlock 256           \\
ResBlock Down 256      \\
2 $\times$ ResBlock 256           \\
Global Mean Pooling    \\
Dense $\rightarrow$ 1  \\ \hline
\end{tabular}
\end{minipage}
\begin{minipage}{.24\textwidth}
\begin{tabular}{c}
CelebA 64             \\ \hline
3$\times$3 Conv2d, 64 \\
ResBlock Down 64      \\
ResBlock Down 128     \\
ResBlock 128          \\
ResBlock Down 256     \\
ResBlock 256          \\
ResBlock Down 256     \\
ResBlock 256          \\
Global Mean Pooling   \\
Dense $\rightarrow$ 1 \\ \hline
\end{tabular}
\end{minipage}
\begin{minipage}{.24\textwidth}
\begin{tabular}{c}
LSUN Church 128       \\ \hline
3$\times$3 Conv2d, 64 \\
ResBlock Down 64      \\
ResBlock Down 128     \\
ResBlock Down 128     \\
ResBlock 128          \\
ResBlock Down 256     \\
ResBlock 256          \\
ResBlock Down 256     \\
ResBlock 256          \\
Global Mean Pooling   \\
Dense $\rightarrow$ 1 \\ \hline
\end{tabular}
\end{minipage}
\begin{minipage}{.24\textwidth}
\begin{tabular}{c}
Cat2dog 64            \\ \hline
3$\times$3 Conv2d, 64 \\
ResBlock Down 64      \\
ResBlock Down 128     \\
Self Attention 128    \\
ResBlock Down 128     \\
ResBlock Down 256     \\
Global Mean Pooling   \\
Dense $\rightarrow$ 1 \\ \hline
\end{tabular}
\end{minipage}
\end{table*}

\paragraph{Unconditional Image Generation.}
For all datasets, we perform random horizontal flip as a form of data augmentation and normalize pixel values to the range [-1, 1]. For CelebA, we center crop the image to 140$\times$140 before resizing to 64$\times$64. For LSUN Church, we resize the image to 128$\times$128 directly. We use the same training hyperparameters across all three datasets, which is as follows. We train all models with 100000 training steps with the Adam optimizer with a batch size of 64, excepet for CIFAR10 with uniform prior where we trained for 120000 steps. We use a learning rate of $1\times10^{-4}$ and decay twice by a factor of 0.1 when there are 20000 and 10000 remaining steps. We set the number of flow steps to $K=100$ at training time. When sampling, we set $\kappa=20$ for all DDP experiments and $\kappa=10$ for UP experiments. Other flow hyperparameters include a step size of $\eta=3$ and noise factor $\nu=10^{-2}$. The specific ResNet architectures are given in Table \ref{tab:archi}. We update model weights using an exponential moving average~\cite{song2020improved} given by $\bm{\theta}' \leftarrow m\bm{\theta}' + (1-m)\bm{\theta}_i$, where $\bm{\theta}_i$ is the parameters of the model at the $i$-th training step, and $\bm{\theta}'$ is an independent copy of the parameters that we save and use for evaluation. We set $m=0.998$. We use the LeakyReLU activation for our networks with a negative slope of 0.2. We experimented with spectral normalization and self attention layers for unconditional image generation, but found that training was stable enough such that they were not worth the added computational cost. The FID results in Table \ref{tab:fid} are obtained by generating 50000 images, and testing the results against the training set for both CIFAR10 and CelebA.

\paragraph{Conditional Generation with Robust Classifier.}
The unconditional model used for conditional generation is the same model obtained from the section above. The pretrained robust classifier checkpoint is obtained from the \texttt{robustness}\footnote{https://github.com/MadryLab/robustness} Python library~\cite{robustness}. It is based on a ResNet50 architecture and is trained with $L_2$-norm perturbations of $\varepsilon=1$. 

We choose the LR-KL variant for our results in Fig. \ref{fig:class_conditional}. This means that our conditional flow is given by
\begin{align}
    \rvx_{k+1} &= \rvx_{k} - 2\eta \nabla_{\rvx}\log\left(r_\theta(\rvx_{k}) *  \frac{1}{N} p(y=n|\rvx_{k})^{-1}\right) + \nu\bm{\xi}_{k} \\
    &= \rvx_{k} - 2\eta \nabla_{\rvx}\left( \log r_\theta(\rvx_{k}) -  \phi \log p(y=n|\rvx_{k})\right) + \nu\bm{\xi}_{k}
\end{align}
where in the second line we introduce $\phi$ as a parameter that scales the magnitude of the classifier's gradients so they are comparable to the magnitude of \modelname's gradients. We use $\phi=0.1$.

\paragraph{Unpaired Image-to-image Translation.}
The Cat2dog dataset contains 871 Birman cat images and 1364 Samoyed and Husky dog images. 100 of each are set aside as test images. We first resize the images to 84$
\times$84 before center cropping to 64$\times$64. Due to the relatively small size of the dataset, we use a shallower ResNet architecture as compared to CelebA $64^2$ despite the same resolution (Table \ref{tab:archi}) to prevent overfitting. We also utilize spectral normalization and a self attention layer at the 128-channel level to further boost stability. We set $K=100$ during training and $K=110$ at test time. As we observed the model tends to diverge late in training, we limit the number of training steps to 40000, with a decay factor of 0.1 applied to the learning rate at steps 20000 and 30000. All other hyperparameters are kept identical to the experiments on unconditional image generation. We report results for LSIF-$\chi^2$, although we have experimented with LR-KL and found performance to be similar. The FID result in Table. \ref{tab:im2im-fid} is obtained by translating the 100 test cat images, and testing the results against the 100 test dog images.

\begin{table*}
\centering
\caption{FID scores for image-to-image translation with the Cat2dog dataset.}
\label{tab:im2im-fid}
\begin{tabular}{lc}
\hline
Model                     & FID $\downarrow$ \\ \hline
\modelname{} & 108.10           \\
CycleGAN                  & 51.79            \\ \hline
\end{tabular}
\end{table*}

\begin{figure*}
     \centering
     \begin{subfigure}[b]{0.32\textwidth}
         \centering
         \includegraphics[width=\textwidth]{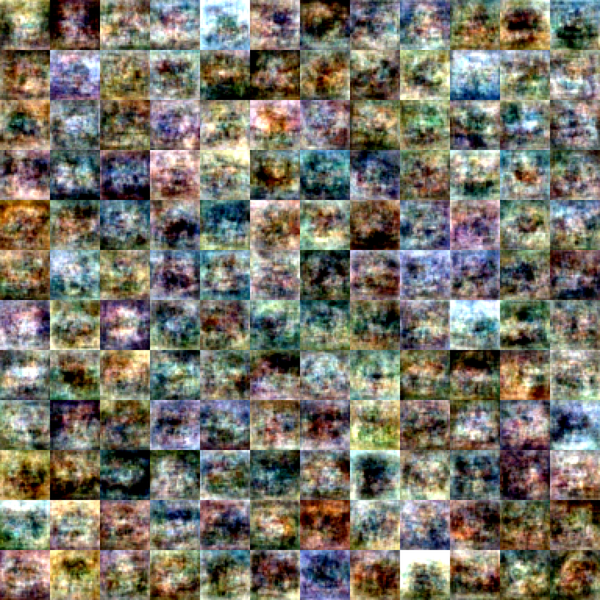}
         \caption{}
         \label{fig:cifar10_learned_prior}
     \end{subfigure}
     \hfill
     \begin{subfigure}[b]{0.32\textwidth}
         \centering
         \includegraphics[width=\textwidth]{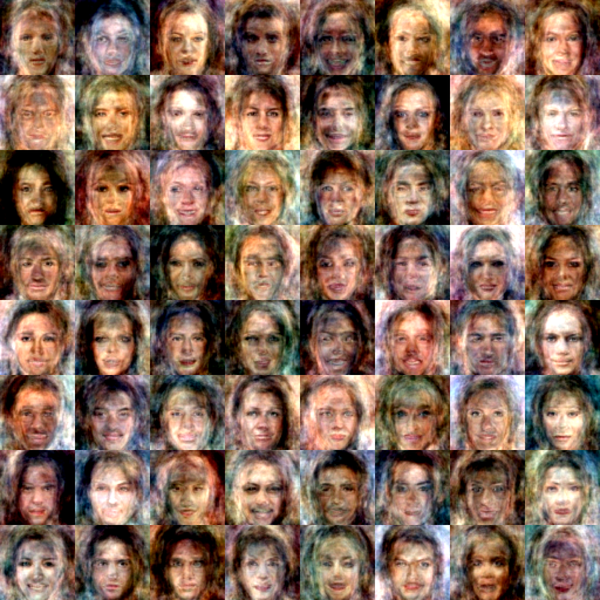}
         \caption{}
         \label{fig:celeba_learned_prior}
     \end{subfigure}
    \begin{subfigure}[b]{0.32\textwidth}
         \centering
         \includegraphics[width=\textwidth]{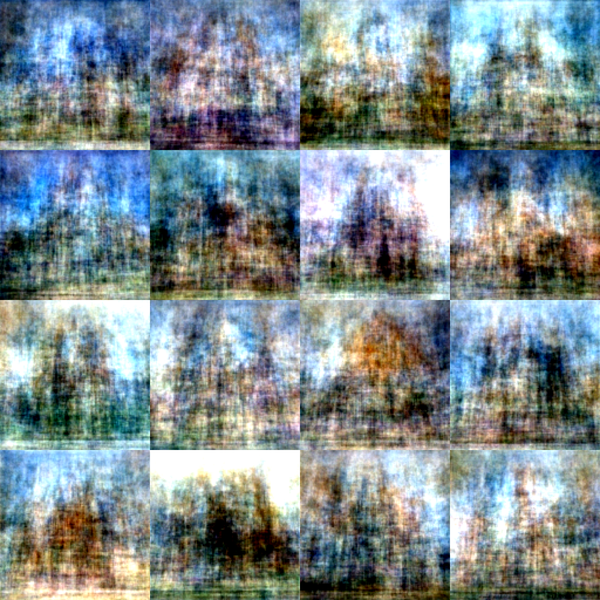}
         \caption{}
         \label{fig:lsun_church_learned_prior}
     \end{subfigure}
        \caption{Samples drawn from the data-dependent priors of (a) CIFAR10 $32^2$, (b) CelebA $64^2$ and (c) LSUN Church $128^2$.}
        \label{fig:learned_prior}
\end{figure*}

\clearpage
\newpage
\onecolumn

\section{Nearest Neighbors}\label{app:nn}
\begin{figure}[!htb]
    \label{fig:nn_cifar}
     \centering
     \begin{subfigure}[b]{0.49\textwidth}
         \centering
         \includegraphics[width=\textwidth]{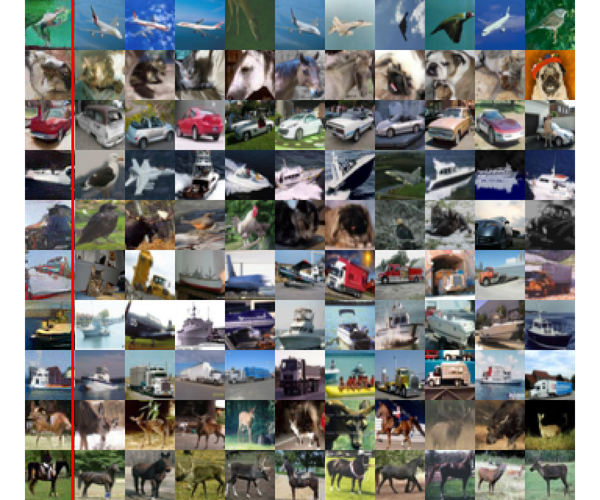}
         \caption{}
     \end{subfigure}
     \hfill
     \begin{subfigure}[b]{0.49\textwidth}
         \centering
         \includegraphics[width=\textwidth]{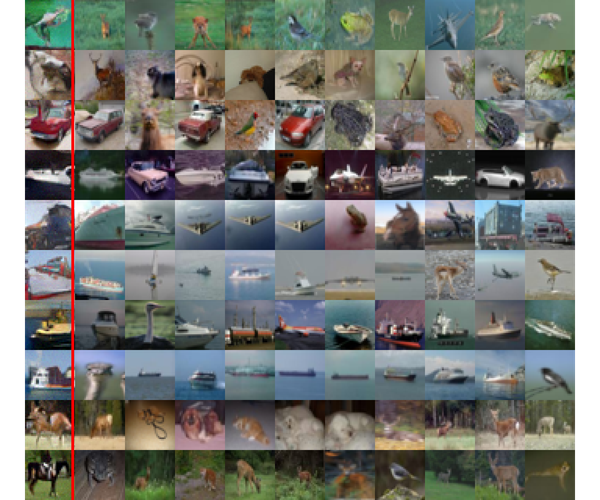}
         \caption{}
     \end{subfigure}
    \caption{Nearest neighbor images for CIFAR10 as measured by $L_2$ distance in (a) the feature space of an Inception V3 network pretrained on ImageNet and (b) data space. The column to the left of the red line are samples from \modelname{} LSIF-$\chi^2$. The images to the right of the line are the 10 nearest neighbors in the training dataset.}
\end{figure}

\begin{figure}[!htb]
    \label{fig:nn_celeba}
     \centering
     \begin{subfigure}[b]{0.49\textwidth}
         \centering
         \includegraphics[width=\textwidth]{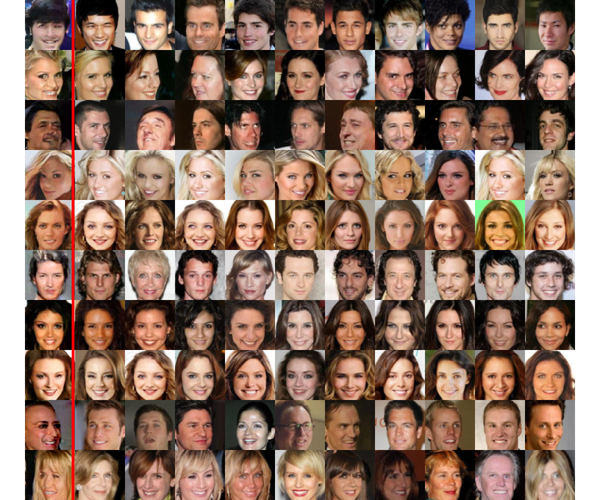}
         \caption{}
     \end{subfigure}
     \hfill
     \begin{subfigure}[b]{0.49\textwidth}
         \centering
         \includegraphics[width=\textwidth]{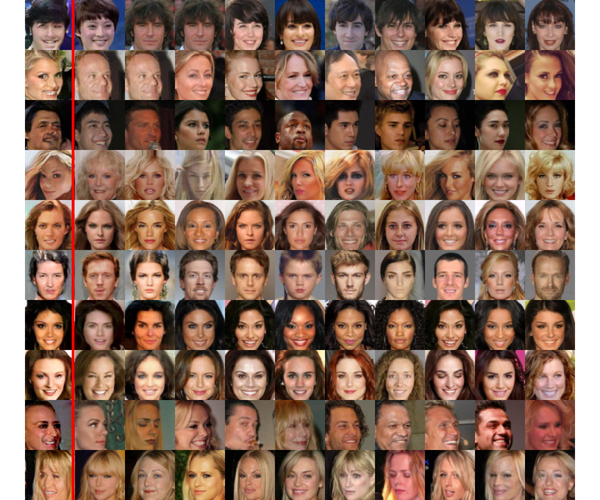}
         \caption{}
     \end{subfigure}
    \caption{Nearest neighbor images for CelebA as measured by $L_2$ distance in (a) the feature space of an Inception V3 network pretrained on ImageNet and (b) data space. The column to the left of the red line are samples from \modelname{} LSIF-$\chi^2$. The images to the right of the line are the 10 nearest neighbors in the training dataset.}
\end{figure}

\clearpage
\newpage

\section{Uncurated Samples \modelname-DDP}\label{app:ddp_samples}

\begin{figure}[!htb]
    \centering
    \includegraphics[width=\textwidth, trim={0, 0cm, 0, 0cm}, clip]{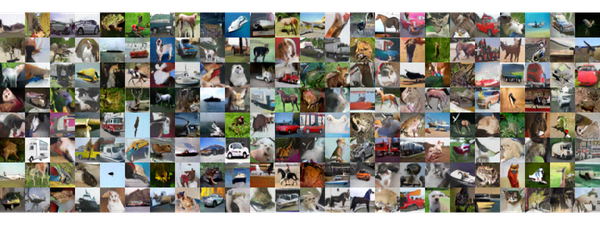}
    \caption{Uncurated samples of CIFAR10 LSIF-$
    \chi^2$.}
    \label{fig:full_LSIF_Pearson_cifar10}
\end{figure}

\begin{figure}[!htb]
    \centering
    \includegraphics[width=\textwidth,trim={0, 0cm, 0, 0cm}, clip]{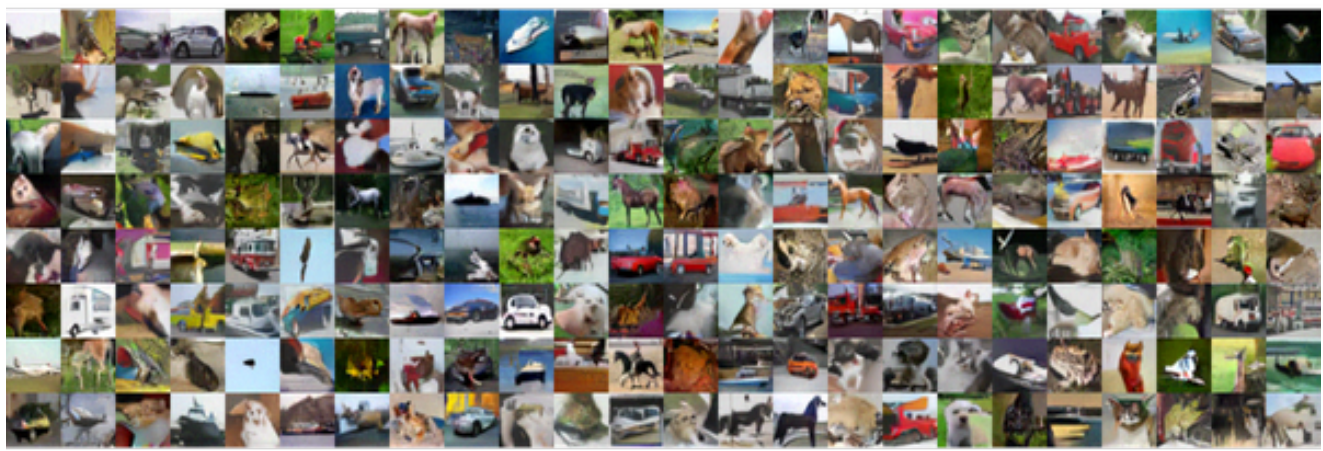}
    \caption{Uncurated samples of CIFAR10 LR-KL.}
    \label{fig:full_LR_KL_cifar10.png}
\end{figure}

\begin{figure}[!htb]
    \centering
    \includegraphics[width=\textwidth, trim={0, 0cm, 0, 0cm}, clip]{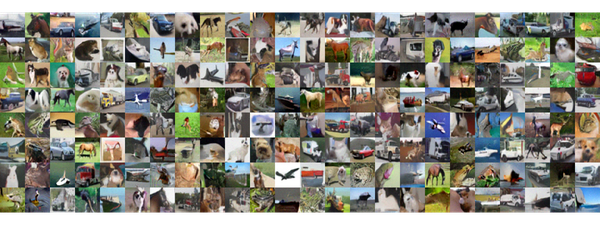}
    \caption{Uncurated samples of CIFAR10 LR-JS.}
    \label{fig:full_cifar_lr_js}
\end{figure}

\begin{figure}[!htb]
    \centering
    \includegraphics[width=\textwidth, trim={0, 0cm, 0, 0cm}, clip]{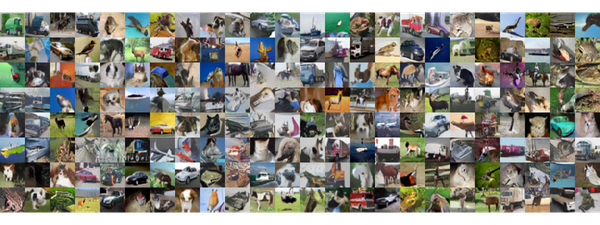}
    \caption{Uncurated samples of CIFAR10 LR-logD.}
    \label{fig:full_cifar_lr_logd}
\end{figure}

\begin{figure}[!htb]
    \centering
    \includegraphics[width=\textwidth, trim={0, 0cm, 0, 0cm}, clip]{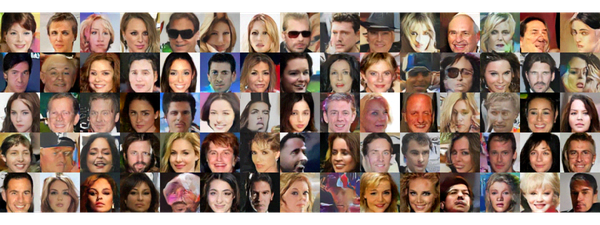}
    \caption{Uncurated samples of CelebA LSIF-$\chi^2$.}
    \label{fig:full_celeba_lsif_pearson}
\end{figure}

\begin{figure}[!htb]
    \centering
    \includegraphics[width=\textwidth, trim={0, 0cm, 0, 0cm}, clip]{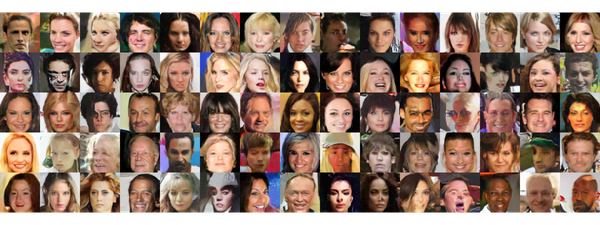}
    \caption{Uncurated samples of CelebA LR-KL.}
    \label{fig:full_celeba_lr_kl}
\end{figure}

\begin{figure}[!htb]
    \centering
    \includegraphics[width=\textwidth, trim={0, 0cm, 0, 0cm}, clip]{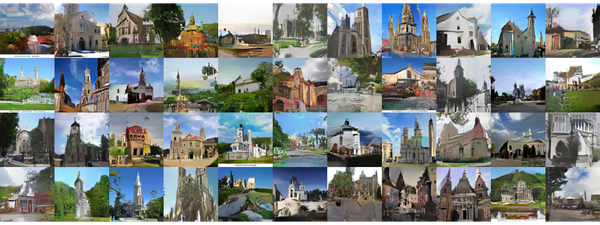}
    \caption{Uncurated samples of LSUN Church LSIF-$\chi^2$.}
    \label{fig:full_church_lsif_pearson}
\end{figure}

\clearpage
\newpage

\section{Uncurated Samples \modelname-UP}\label{app:up_samples}
\begin{figure}[!htb]
     \centering
     \begin{subfigure}[b]{0.49\textwidth}
         \centering
         \includegraphics[width=\textwidth]{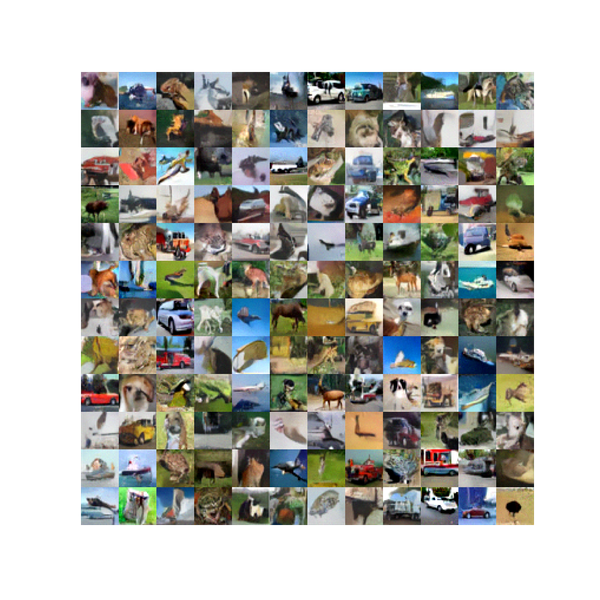}
         \caption{LSIF-$\chi^2$ uniform prior.}
         \label{fig:cifar10_up_lsif_pearson}
     \end{subfigure}
     \hfill
     \begin{subfigure}[b]{0.49\textwidth}
         \centering
         \includegraphics[width=\textwidth]{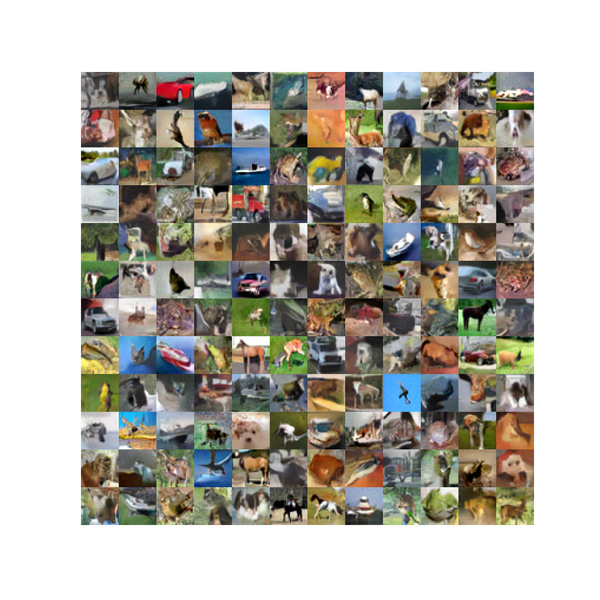}
         \caption{LR-KL uniform prior.}
         \label{fig:cifar10_up_lr_kl}
     \end{subfigure}
        \caption{Uncurated CIFAR10 samples with \modelname-UP.}
        \label{fig:up_samples}
\end{figure}

%% file: official_preprint.bbl
\begin{thebibliography}{10}\itemsep=-1pt

\bibitem{ansari2020refining}
Abdul~Fatir Ansari, Ming~Liang Ang, and Harold Soh.
\newblock Refining deep generative models via discriminator gradient flow.
\newblock In {\em ICLR}, 2021.

\bibitem{arbel2019maximum}
Michael Arbel, Anna Korba, Adil Salim, and Arthur Gretton.
\newblock Maximum mean discrepancy gradient flow.
\newblock {\em Advances in Neural Information Processing Systems}, 32, 2019.

\bibitem{choi2020stargan}
Yunjey Choi, Youngjung Uh, Jaejun Yoo, and Jung-Woo Ha.
\newblock Stargan v2: Diverse image synthesis for multiple domains.
\newblock In {\em Proceedings of the IEEE/CVF conference on computer vision and
  pattern recognition}, pages 8188--8197, 2020.

\bibitem{dhariwal2021diffusion}
Prafulla Dhariwal and Alexander Nichol.
\newblock Diffusion models beat gans on image synthesis.
\newblock {\em Advances in Neural Information Processing Systems},
  34:8780--8794, 2021.

\bibitem{du2023nonparametric}
Chao Du, Tianbo Li, Tianyu Pang, Shuicheng Yan, and Min Lin.
\newblock Nonparametric generative modeling with conditional and
  locally-connected sliced-wasserstein flows.
\newblock {\em arXiv preprint arXiv:2305.02164}, 2023.

\bibitem{du2019implicit}
Yilun Du and Igor Mordatch.
\newblock Implicit generation and generalization in energy-based models.
\newblock {\em arXiv preprint arXiv:1903.08689}, 2019.

\bibitem{robustness}
Logan Engstrom, Andrew Ilyas, Shibani Santurkar, and Dimitris Tsipras.
\newblock Robustness (python library), 2019.

\bibitem{fan2021variational}
Jiaojiao Fan, Amirhossein Taghvaei, and Yongxin Chen.
\newblock Variational wasserstein gradient flow.
\newblock {\em arXiv preprint arXiv:2112.02424}, 2021.

\bibitem{franceschi2024unifying}
Jean-Yves Franceschi, Mike Gartrell, Ludovic Dos~Santos, Thibaut Issenhuth,
  Emmanuel de B{\'e}zenac, Micka{\"e}l Chen, and Alain Rakotomamonjy.
\newblock Unifying gans and score-based diffusion as generative particle
  models.
\newblock {\em Advances in Neural Information Processing Systems}, 36, 2024.

\bibitem{gao2022deep}
Yuan Gao, Jian Huang, Yuling Jiao, Jin Liu, Xiliang Lu, and Zhijian Yang.
\newblock Deep generative learning via euler particle transport.
\newblock In {\em Mathematical and Scientific Machine Learning}, pages
  336--368. PMLR, 2022.

\bibitem{gao2019deep}
Yuan Gao, Yuling Jiao, Yang Wang, Yao Wang, Can Yang, and Shunkang Zhang.
\newblock Deep generative learning via variational gradient flow.
\newblock In {\em International Conference on Machine Learning}, pages
  2093--2101. PMLR, 2019.

\bibitem{goodfellow2014generative}
Ian Goodfellow, Jean Pouget-Abadie, Mehdi Mirza, Bing Xu, David Warde-Farley,
  Sherjil Ozair, Aaron Courville, and Yoshua Bengio.
\newblock Generative adversarial nets.
\newblock {\em Advances in neural information processing systems}, 27, 2014.

\bibitem{gulrajani2017improved}
Ishaan Gulrajani, Faruk Ahmed, Martin Arjovsky, Vincent Dumoulin, and Aaron~C
  Courville.
\newblock Improved training of wasserstein gans.
\newblock {\em Advances in neural information processing systems}, 30, 2017.

\bibitem{ho2020denoising}
Jonathan Ho, Ajay Jain, and Pieter Abbeel.
\newblock Denoising diffusion probabilistic models.
\newblock {\em Advances in Neural Information Processing Systems},
  33:6840--6851, 2020.

\bibitem{jordan1998variational}
Richard Jordan, David Kinderlehrer, and Felix Otto.
\newblock The variational formulation of the fokker--planck equation.
\newblock {\em SIAM journal on mathematical analysis}, 29(1):1--17, 1998.

\bibitem{lecun2006tutorial}
Yann LeCun, Sumit Chopra, Raia Hadsell, M Ranzato, and F Huang.
\newblock A tutorial on energy-based learning.
\newblock {\em Predicting structured data}, 1(0), 2006.

\bibitem{lee2018diverse}
Hsin-Ying Lee, Hung-Yu Tseng, Jia-Bin Huang, Maneesh Singh, and Ming-Hsuan
  Yang.
\newblock Diverse image-to-image translation via disentangled representations.
\newblock In {\em Proceedings of the European conference on computer vision
  (ECCV)}, pages 35--51, 2018.

\bibitem{liu2019understanding}
Chang Liu, Jingwei Zhuo, and Jun Zhu.
\newblock Understanding mcmc dynamics as flows on the wasserstein space.
\newblock In {\em International Conference on Machine Learning}, pages
  4093--4103. PMLR, 2019.

\bibitem{liutkus2019sliced}
Antoine Liutkus, Umut Simsekli, Szymon Majewski, Alain Durmus, and
  Fabian-Robert St{\"o}ter.
\newblock Sliced-wasserstein flows: Nonparametric generative modeling via
  optimal transport and diffusions.
\newblock In {\em International Conference on Machine Learning}, pages
  4104--4113. PMLR, 2019.

\bibitem{miyato2018spectral}
Takeru Miyato, Toshiki Kataoka, Masanori Koyama, and Yuichi Yoshida.
\newblock Spectral normalization for generative adversarial networks.
\newblock {\em arXiv preprint arXiv:1802.05957}, 2018.

\bibitem{mokrov2021large}
Petr Mokrov, Alexander Korotin, Lingxiao Li, Aude Genevay, Justin~M Solomon,
  and Evgeny Burnaev.
\newblock Large-scale wasserstein gradient flows.
\newblock {\em Advances in Neural Information Processing Systems},
  34:15243--15256, 2021.

\bibitem{mroueh2021convergence}
Youssef Mroueh and Truyen Nguyen.
\newblock On the convergence of gradient descent in gans: Mmd gan as a gradient
  flow.
\newblock In {\em International Conference on Artificial Intelligence and
  Statistics}, pages 1720--1728. PMLR, 2021.

\bibitem{mroueh2019sobolev}
Youssef Mroueh, Tom Sercu, and Anant Raj.
\newblock Sobolev descent.
\newblock In {\em The 22nd International Conference on Artificial Intelligence
  and Statistics}, pages 2976--2985. PMLR, 2019.

\bibitem{nie2021controllable}
Weili Nie, Arash Vahdat, and Anima Anandkumar.
\newblock Controllable and compositional generation with latent-space
  energy-based models.
\newblock {\em Advances in Neural Information Processing Systems},
  34:13497--13510, 2021.

\bibitem{nijkamp2019learning}
Erik Nijkamp, Mitch Hill, Song-Chun Zhu, and Ying~Nian Wu.
\newblock Learning non-convergent non-persistent short-run mcmc toward
  energy-based model.
\newblock {\em Advances in Neural Information Processing Systems}, 32, 2019.

\bibitem{pang2020learning}
Bo Pang, Tian Han, Erik Nijkamp, Song-Chun Zhu, and Ying~Nian Wu.
\newblock Learning latent space energy-based prior model.
\newblock {\em Advances in Neural Information Processing Systems},
  33:21994--22008, 2020.

\bibitem{rhodes2020telescoping}
Benjamin Rhodes, Kai Xu, and Michael~U Gutmann.
\newblock Telescoping density-ratio estimation.
\newblock {\em Advances in neural information processing systems},
  33:4905--4916, 2020.

\bibitem{risken1985fokker}
Hannes Risken and JH Eberly.
\newblock The fokker-planck equation, methods of solution and applications.
\newblock {\em Journal of the Optical Society of America B Optical Physics},
  2(3):508, 1985.

\bibitem{santambrogio2017euclidean}
Filippo Santambrogio.
\newblock $\{$Euclidean, metric, and Wasserstein$\}$ gradient flows: an
  overview.
\newblock {\em Bulletin of Mathematical Sciences}, 7(1):87--154, 2017.

\bibitem{santurkar2019image}
Shibani Santurkar, Andrew Ilyas, Dimitris Tsipras, Logan Engstrom, Brandon
  Tran, and Aleksander Madry.
\newblock Image synthesis with a single (robust) classifier.
\newblock {\em Advances in Neural Information Processing Systems}, 32, 2019.

\bibitem{song2019generative}
Yang Song and Stefano Ermon.
\newblock Generative modeling by estimating gradients of the data distribution.
\newblock {\em Advances in Neural Information Processing Systems}, 32, 2019.

\bibitem{song2020improved}
Yang Song and Stefano Ermon.
\newblock Improved techniques for training score-based generative models.
\newblock {\em Advances in neural information processing systems},
  33:12438--12448, 2020.

\bibitem{song2020score}
Yang Song, Jascha Sohl-Dickstein, Diederik~P Kingma, Abhishek Kumar, Stefano
  Ermon, and Ben Poole.
\newblock Score-based generative modeling through stochastic differential
  equations.
\newblock {\em arXiv preprint arXiv:2011.13456}, 2020.

\bibitem{sugiyama2012density}
Masashi Sugiyama, Taiji Suzuki, and Takafumi Kanamori.
\newblock {\em Density ratio estimation in machine learning}.
\newblock Cambridge University Press, 2012.

\bibitem{sugiyama2012densityratio}
Masashi Sugiyama, Taiji Suzuki, and Takafumi Kanamori.
\newblock Density-ratio matching under the bregman divergence: a unified
  framework of density-ratio estimation.
\newblock {\em Annals of the Institute of Statistical Mathematics},
  64(5):1009--1044, 2012.

\bibitem{xiao2020vaebm}
Zhisheng Xiao, Karsten Kreis, Jan Kautz, and Arash Vahdat.
\newblock Vaebm: A symbiosis between variational autoencoders and energy-based
  models.
\newblock {\em arXiv preprint arXiv:2010.00654}, 2020.

\bibitem{xie2018cooperative}
Jianwen Xie, Yang Lu, Ruiqi Gao, and Ying~Nian Wu.
\newblock Cooperative learning of energy-based model and latent variable model
  via mcmc teaching.
\newblock In {\em Proceedings of the AAAI Conference on Artificial
  Intelligence}, volume~32, 2018.

\bibitem{xie2016theory}
Jianwen Xie, Yang Lu, Song-Chun Zhu, and Yingnian Wu.
\newblock A theory of generative convnet.
\newblock In {\em International Conference on Machine Learning}, pages
  2635--2644. PMLR, 2016.

\bibitem{zhao2021unpaired}
Yang Zhao and Changyou Chen.
\newblock Unpaired image-to-image translation via latent energy transport.
\newblock In {\em Proceedings of the IEEE/CVF conference on computer vision and
  pattern recognition}, pages 16418--16427, 2021.

\bibitem{zhu2017unpaired}
Jun-Yan Zhu, Taesung Park, Phillip Isola, and Alexei~A Efros.
\newblock Unpaired image-to-image translation using cycle-consistent
  adversarial networks.
\newblock In {\em Proceedings of the IEEE international conference on computer
  vision}, pages 2223--2232, 2017.

\end{thebibliography}
